\documentclass{article} % For LaTeX2e
\usepackage{times}
\usepackage{picinpar}
\usepackage{url}
\usepackage{amsfonts,amssymb}
\usepackage{amssymb}
\usepackage{graphicx}
\usepackage[ruled,noend]{algorithm2e}
\usepackage{algorithmic}
\usepackage{graphicx}
\usepackage{epstopdf}
\usepackage{subfigure}
\usepackage{epsfig}
\usepackage{amsmath}
\usepackage{epsf}
\usepackage{wrapfig}
\usepackage{url}
\usepackage[margin=1in]{geometry}
\usepackage{lipsum}% http://ctan.org/pkg/lipsum
\usepackage{listings}% http://ctan.org/pkg/listings
\usepackage{caption}%

\newtheorem{theorem}{Theorem}

\newtheorem{prop}{Proposition}
\newtheorem{proof}{Proof}

\title{\huge
  Seeing the Forest from the Trees in Two Looks:
  \\
  { Matrix Sketching by Cascaded Bilateral Sampling }
}

\author{Kai Zhang$^1$, Chuanren Liu$^2$, Jie Zhang$^3$, Hui Xiong$^4$, Eric Xing$^5$, Jieping Ye$^6$\\
$^1$NEC Laboratories Amercia, Princeton\\
$^2$Lebow College of Business, Drexel University, Philidelphia\\
$^3$Center of Computational Biology, Fudan University China\\
$^4$Management Science \& Information Systems Department, Rutgers\\
$^5$Machine Learning Department, Carnegie Mellon University\\
$^6$Department of Computational Medicine and Bioinformatics\\ University of Michigan, Ann Arbor
}

\def\x{\mathbf{x}}
\def\y{\mathbf{y}}
\def\C{\mathbf{C}}
\def\U{\mathbf{U}}
\def\R{\mathbf{R}}
\def\r{\mathbb{R}}
\def\W{\mathbf{W}}
\def\M{\mathbf{M}}
\def\S{\mathbf{S}}
\def\I{\mathcal{I}}
\def\V{\mathbf{V}}

\def\P{\mathbf{P}}
\def\c{\theta}
\def\T{{T}}
\def\f{{\phi}}
\def\X{\mathbf{X}}
\def\Y{\mathbf{Y}}
\def\ZZ{\mathbf{Z}}
\def\SS{\mathbf{\Sigma}}
\def\Q{\mathbf{Q}}
\def\Y{\mathbf{Y}}
\def\Omg{\mathbf{\Omega}}

\def\Z{\mathcal{Z}}

\def\A{\mathbf{A}}
\def\R{\mathbf{R}}
\begin{document}

\maketitle

\begin{abstract}

Matrix sketching is aimed at finding close approximations of a matrix by factors of much smaller dimensions, which has important applications in optimization and machine learning. Given a matrix $\A\in \mathbb{R}^{m\times n}$, state-of-the-art randomized algorithms take $\mathcal{O}(m\cdot n)$ time and space to obtain its low-rank decomposition. Although quite useful, the need to store or manipulate the entire matrix makes it a computational bottleneck for truly large and dense inputs. Can we sketch an $m$-by-$n$ matrix in $\mathcal{O}(m+n)$ cost by accessing only a small fraction of its rows and columns, without knowing anything about the remaining data? In this paper, we propose the \underline{ca}scaded \underline{b}ilateral \underline{s}ampling (CABS) framework to solve this problem.
We start from demonstrating how the approximation quality of bilateral matrix sketching depends on the encoding powers of sampling. In particular, the sampled rows and columns should correspond to the code-vectors in the ground truth decompositions. Motivated by this analysis, we propose to first generate a pilot-sketch using simple random sampling, and then pursue more advanced, ``follow-up'' sampling on the pilot-sketch factors seeking maximal encoding powers. In this cascading process, the rise of approximation quality is shown to be lower-bounded by the improvement of encoding powers in the follow-up sampling step, thus theoretically guarantees the algorithmic boosting property.
Computationally, our framework only takes linear time and space, and at the same time its performance rivals the quality of state-of-the-art algorithms consuming a quadratic amount of resources.
%In the meantime, it is quite friendly to parallel architectures and can provably reduce the communication overhead to a constant in parallel low-rank matrix decomposition. %We further extend our framework to  parallel environment and come up with a novel paradigm that can provably reduce the communication overhead to a constant in parallel low-rank matrix decomposition.
Empirical evaluations on benchmark data fully demonstrate the potential of our methods in large scale matrix sketching and related areas.
\end{abstract}

\section{Introduction}

Matrix sketching is aimed at finding close approximations of a matrix by using factors of much smaller dimensions, which plays important roles in optimization and machine learning \cite{nmf,svm,fisher,spectral,pnas,review,cur,random}. A promising tool to solve this problem is low-rank matrix decomposition, which approximates an input matrix $\A\in\r^{m\times n}$ by  $\A \approx \P\Q^\top$, where $\P$ and $\Q$ have a low column rank $k\ll m,n$.
%Low-rank approximation appears in many scientific applications and numerical computing procedures, such as principal component analysis, linear system of equations, and even solving PDEs.  %It  can greatly reduce the memory and computational cost in matrix multiplication, pseudo-inverse, or least squares.
Recent advances in randomized algorithms have made it state-of-the-art in low-rank matrix sketching or decomposition. For example, Frieze \emph{et al.} \cite{fastmc} and Drineas \emph{et al.}  \cite{montecarlo} proposed to use monte-carlo sampling to select informative rows or columns from a matrix; Mahoney and Drineas \cite{cur} proposed CUR matrix decomposition and used {\em statistical leverage scores} to perform sampling;
Halko \emph{et al.} \cite{random} proposed to project the matrix to a lower-dimensional space and then compute the desired factorization \cite{fastleverage,FCT}.
%\citet{yang2015explicit} proposed to optimize the sampling performances based on the square root of the statistical leverage score.

These algorithms compute approximate decompositions in $\mathcal{O}(m\cdot n)$ time and space, which is more efficient than a singular value decomposition using $\mathcal{O}(n^2m)$ time and $\mathcal{O}(m\cdot n)$ space (if $m\geq n$). However, the whole input matrix must be fully involved in the computations, either in computing high-quality sampling probabilities \cite{montecarlo,cur,fastmc,yang2015explicit}, or being compressed into a lower-dimensional space \cite{random}. This can lead to potential computational and memory bottlenecks in particular  for truly large and dense matrices.

Is it possible to sketch an $m$-by-$n$ matrix in  $\mathcal{O}(m+n)$ time and space, by using only a small number of its rows and columns and never knowing the remaining entries? To the best of our knowledge, this is still an open problem. Actually, the challenge may even seem unlikely to resolve at first sight, because by accessing such a small fraction of the data, resultant approximation can be quite inaccurate; besides, a linear amount of resources could hardly afford any advanced sampling scheme other than the uniform sampling; finally, approximation of a general, rectangular matrix with linear cost is much harder than that of a positive semi-definite (PSD) matrix such as kernel matrix \cite{nystrom_williams,nyspami,nystrom}.

To resolve this problem, in this paper we propose a \underline{ca}scaded \underline{b}ilateral \underline{s}ampling (CABS) framework. Our theoretical foundation is an innovative analysis revealing how the approximation quality of bilateral matrix sketching is associated with the encoding powers of sampling. In particular, selected columns and rows should correspond to representative code-vectors in the ground-truth embeddings. Motivated by this, we propose to first generate a pilot-sketch using simple random sampling, and then pursue more advanced exploration on this pilot-sketch/embedding seeking maximal encoding powers.
In this process, the rise of approximation quality is shown to be lower-bounded by the improvement of encoding powers through the follow-up sampling, thus theoretically guarantees the algorithmic boosting property. Computationally, both rounds of sampling-and-sketching operations require only a linear cost; however, when cascaded properly, the performance rivals the quality of state-of-the-art algorithms consuming a quadratic amount of resources.
The CABS framework is highly memory and pass efficient by only accessing twice a small number of specified rows and columns, thus quite suitable to large and dense matrices which won't fit in the main memory. In the meantime, the sketching results are quite easy to interpret.   %In the meantime, it quite friendly to parallel architectures and can provably reduce the communication cost to a constant in parallel matrix approximation tasks. % , due to the capacity of quickly locating a small number of signature rows and columns of a matrix. %partitions, the CABS framework can be used to design a novel information sharing mechanism, which only requires a constant amount of communication overheads in parallel processing environment.

%The rest of the paper is organized as follows. Section 2 briefly reviews existing work in low-rank matrix sketching and decomposition. Section 3 presents a novel, heuristic error bound analysis, as well as its indication of a new sampling scheme. Section 4 introduces the cascaded bilateral sampling (CABS) framework, and proves its algorithmic boosting property. Section 5 reports experimental results, and the last section concludes the paper.

\section{Related Work}

%Low-rank approximation appears in many scientific applications and numerical computing procedures, such as principal component analysis, linear system of equations, and even solving PDEs. Given an input matrix $\A\in\r^{m\times n}$, we are interested in the decomposition $\label{eq:lowrank}
%\A \approx \P\Q^\top%  {\underset{m\times k}{\P}} \cdot {\underset{k\times n}{\Q}}
%$ where $\P$ and $\Q$ has a low column rank $k\ll m,n$. %It  can greatly reduce the memory and computational cost in matrix multiplication, pseudo-inverse, or least squares.

\subsection{Quadratic-Cost Algorithms}
We first review state-of-the-art randomized algorithms. They typically consume quadratic, $\mathcal{O}(mn)$ time and space due to the need to access and manipulate the entire input matrix in their  calculations.

{Monte-Carlo sampling} method \cite{fastmc,montecarlo} computes approximate singular vectors of a matrix $\A$ by selecting a subset of its columns using non-uniform, near-optimal sampling probabilities. The selected columns are re-scaled by the probabilities, and its rank-$k$ basis $\Q_k$ is then used to obtain the final decomposition. If the probabilities are chosen as $
p_j \geq \beta\cdot{\|\A_{[:,j]}\|^2}/({\sum_{i = 1}^n \|\A_{[:,i]}\|^2}),
$
then with probability at least $1-\delta$, one has $
\left\|\A- \Q_k\Q_k^\top\A\right\|_F^2 \leq \|\A - \A_k\|_F^2 + \epsilon\|\A\|_F^2
$,
where $\A_k$ is the best rank-$k$ approximation, $\epsilon =2\eta/\sqrt{k/\beta c}%\left(\frac{4\eta^2k}{\beta c}\right)^{\frac{1}{2}}
$, with $\beta\leq 1, \eta = 1+\sqrt{(8/\beta)\log(1/\delta)}$. %More work along this direction can be found in \cite{fastmc}.

%needs to be  existing algorithms in low-rank matrix approximation.
{Random projection} methods \cite{random} project a matrix $\A$ into a subspace $\Q\in \R^{m\times k}$ with orthonormal columns such that $\A\approx \Q\Q^\top\A$. Computing $\Q$  requires multiplying $\A$ with a Gaussian test matrix (or random Fourier transform) $\Omg$ with $q$ steps of power iterations, $\Y = (\A\A^\top)^q\A\Omg$, and then computing QR-decomposition $\Y = \Q\R$. Using a template $\Omg\in \r^{m\times(k+p)}$ with over-sampling parameter $p$,
$
\mathbb{E}\left\|\A - \Q\Q^\top\A\right\| = [1 + {4\sqrt{(k+p)\cdot\min(m,n)}}/(p-1)]{\SS}_{k+1},
$
where $\mathbb{E}$ is expectation with $\Omg$, ${\SS}_{k+1}$ is the $(k+1)$th singular value. New projection methods are explored in \cite{fastleverage,FCT}.%It guarantees that the approximation error is within a small polynomial factor of the minimum possible error. %Another choice of $\Omg$ is structured random matrices such as the sub-sampled random Fourier transform¡£ %, which can further speed up the multiplication $\Y = \A\Omg$ from $\O(mnk)$ to $\O(mn\log(k))$ but can be less accurate.

\begin{wrapfigure}{r}{6cm}
\vskip -5mm
\centering
\label{fig:cur}
\includegraphics[width = 6.5cm,height = 3.3cm]{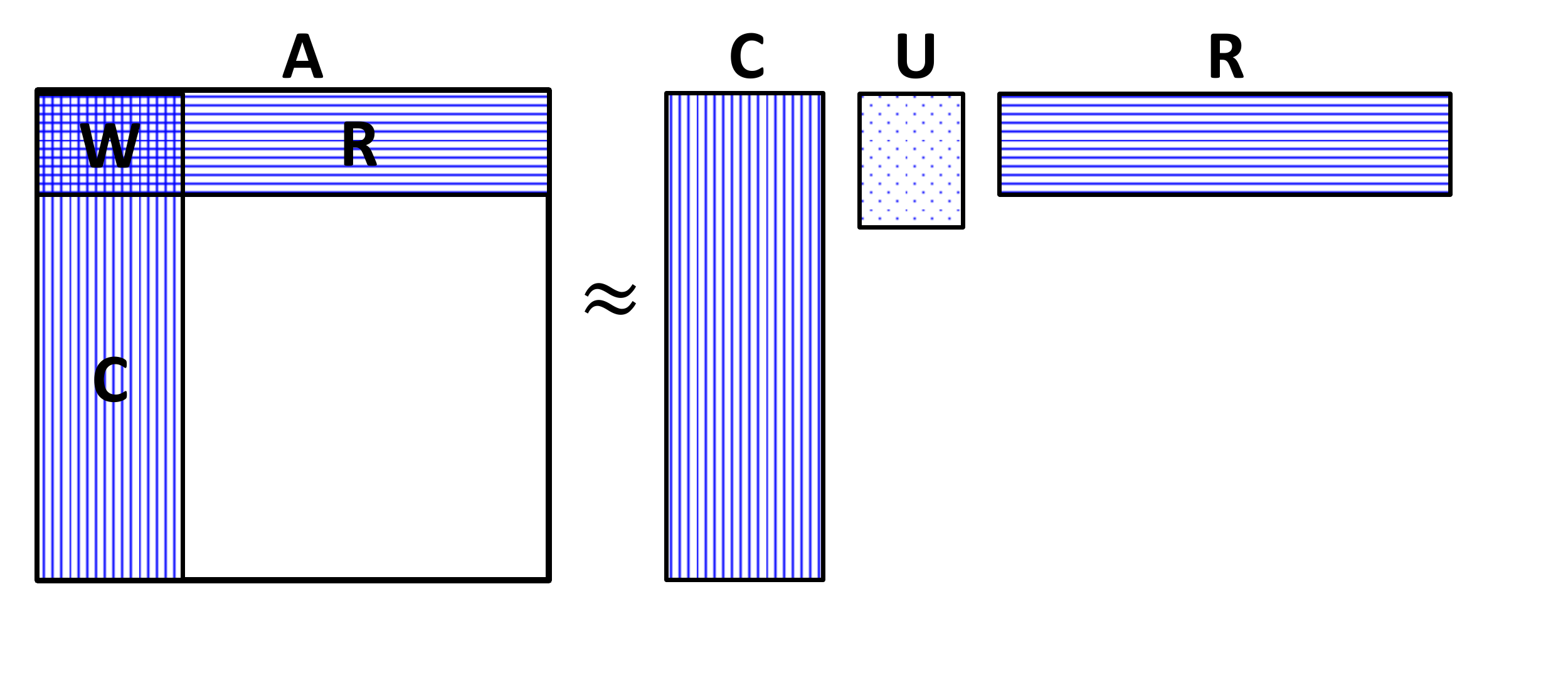}\label{fig:cur}
\vskip -5mm
\caption{Illustration of CUR method.}\label{fig:cur}
\end{wrapfigure}
CUR matrix decomposition \cite{cur} was invented by Mahoney and Drineas to improve the interpretability of low-rank matrix decomposition.
As shown in Figure~\ref{fig:cur}, it samples a subset of columns and rows from $\A$, as $\C\in \r^{m\times k_c}$ and $\R\in \r^{k_r\times n}$, and then solve
\begin{eqnarray}\label{eq:U}
\min_\U \|\A - \C\U\R\|_F^2 \rightarrow \U^* = \C^\dagger \A \R^\dagger,
\end{eqnarray}
where $^\dagger$ is pseudo-inverse.
The CUR method preserves the sparsity
and  non-negativity  properties of input matrices. %, and so the results are much easier to interpret than the singular value decomposition.
If leverage scores are used for sampling, then with high probability,
$\|\A - \C\U\R\|_F \leq (2+\epsilon)\|\A - \A_k\|_F.$
%In \cite{det_leverage}, a deterministic sampling scheme is also explored.
The leverage scores are computed by top-$r$ left/right singular vectors, taking $O(mn)$ space and $O(mnr)$ time. Computing  $\U$ in (\ref{eq:U}) involves multiplying $\A$ with $\C^\dagger$ and $\R^\dagger$. Therefore, even using simple random sampling, CUR may still take  $O(mn)$ time and space.

\subsection{Linear-Cost Algorithms}
The randomized algorithms discussed in Section~2.1 take at least quadratic time and space.  One way to reduce the cost to a linear scale is to restrict the calculations to only a small fraction of rows and columns. In the literature, related work is quite limited.
%Despite the extensive literatures on unilateral sampling, bilateral sampling can be much more difficult. Not only because bilateral sampling only access a significantly smaller proportion of the data, but also because in computing the decompositions, a linear cost constraint prohibits one from using the entire input matrix.
We use a general variant of the CUR method, called \emph{Bilateral Re-sampling CUR} (BR-CUR), to initiate the discussion.%, which is a general structure  subsuming all existing linear-cost algorithms for matrix sketching.

\begin{tabular}{cc}
\raisebox{.01cm}{\begin{minipage}{.53\textwidth}
 \begin{algorithm}[H]
\caption{Bilateral Resampling CUR (BR-CUR)}
{\bf Input}: $\A$, \!base \!sampling {\scriptsize $\I^b_r$, $\I^b_c$}, target \!sampling {\scriptsize $\I^t_r$, $\I^t_c$}\; {\bf Output}: $\C$, ${\U}$, $\R$
\label{alg1}
\begin{algorithmic}[1]
\STATE  Compute bases $\C= \A_{[:,\I^b_c]}$, $\R= \A_{[\I^b_r,:]}$.
\STATE  Sample on bases  $\bar{\C}= \C_{[\I^t_r,:]}$,  $\bar{\R}= \R_{[:,\I^t_c]}$.
\STATE  Compute target block $\M = \A_{[\I^t_r,\I^t_c]}$.
\STATE Solve  $\U^* = \arg\min_\U \|\M - \bar{\C}\U\bar{\R}\|_F^2$.
\STATE Reconstruct by $\A\approx \C\U^*\R$.
\end{algorithmic}
\end{algorithm}
\end{minipage}} &\hskip -4mm
\begin{minipage}{.5\textwidth}
\includegraphics[width = 6.8cm,height = 3.3cm]{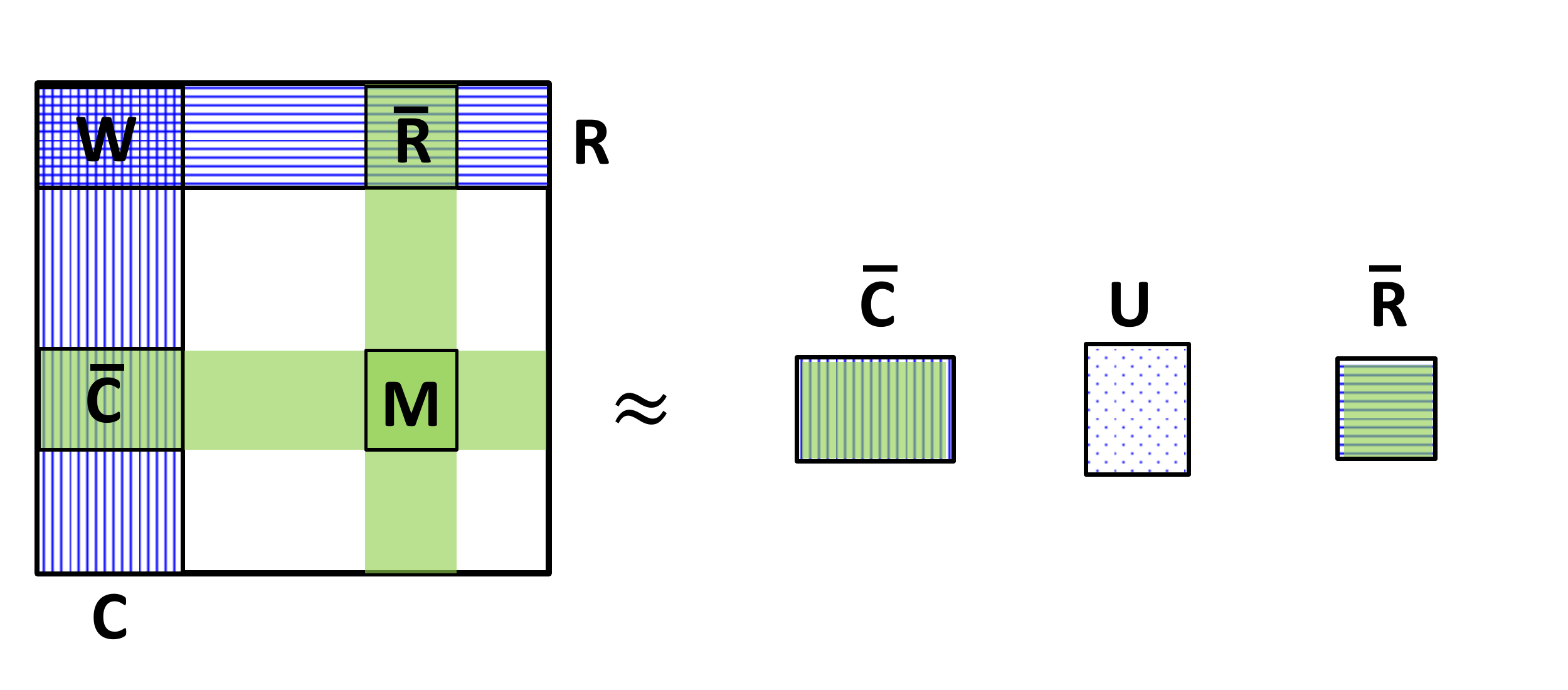}\label{fig:BR-CUR}\vskip -3mm
\captionof{figure}{Illustration of BR-CUR.}
\end{minipage}
\end{tabular}
As illustrated in Figure~2, BR-CUR has two rounds of samplings: blue ``base''-sampling $\I^b_r$ (row) and $\I^b_c$ (column) to construct $\C$ and $\R$, and  green ``target''-sampling $\I^t_r$ (row) and $\I^t_c$ (column) to specify a sub-matrix from $\A$. In computing $\U$ (step~4), it only attempts to minimize the approximation error on the target sub-matrix, therefore being computationally quite efficient. Assume $k_1$ and $k_2$ are the number of random samples selected in the base and target sampling, respectively. Then BR-CUR only takes $\mathcal{O}\left((m+n)(k_1+k_2)\right)$ space and $\mathcal{O}\left((m+n)\max(k_1,k_2)^2\right)$ time.%, which is very efficient.

The BR-CUR procedure subsumes most linear-cost algorithms for matrix sketching in the literature.

\begin{minipage}{1\textwidth}
 1. Nystr\"om method \cite{nystrom_williams,nyspami,nystrom}: in case of symmetric positive semi-definite matrix $\A$, upon using the same base and target sampling, and the same row and column sampling, $\I^b_c = \I^b_r =\I^t_c = \I^t_c$.\vskip .1mm
2. Pseudo-skeleton \cite{skeleton} or bilateral projection \cite{godec}: in case of rectangular matrix $\A$, and using the same base and target sampling, i.e., $\I^b_r = \I^t_r$ and $\I^b_c = \I^t_c$. It generalizes the Nystr\"om method from symmetric to rectangular matrices. Let $\W$ be the intersection of $\C$ and $\R$, then it has compact form,
         \begin{eqnarray}\label{eq:W}
         \A\approx \C\W^\dagger \R.
         \end{eqnarray}\vskip .1mm
         3.
Sketch-CUR \cite{sketchCUR}: in case of rectangular $\A$ and independent base/target sampling (different rates).
\end{minipage}\vskip .0mm
These algorithms only need to access a small fraction of rows and columns to reconstruct the input matrix. However, the performance can be inferior, in particularly on general, rectangular matrices. Full discussions are in  Section~4.2. %of the linear costNext, will focus our theoretic analysis on linear-cost bilateral sketchingusing the pseudo-skeleton method on In particular, we find that the Pseudo-skeleton method is always superior to sketch-CUR method in case base and target sampling rates are chosen the same.
More recently, online streaming algorithms are also designed for matrix sketching \cite{memory_efficient,streaming,streamingPCA}. Their memory cost is much smaller, but the whole matrix still needs to be fully accessed and the time complexity is at least $\mathcal{O}(m\cdot n)$.
For sparse matrices, the rarity of non-zeros entries can be exploited to design algorithms in input sparsity time \cite{sparse}. Note that the method proposed in this paper is applicable to both dense and sparse matrices; in particular, significant performance gains will be observed in both scenarios in our empirical evaluations (Section~5).%performance gains as well.

%\begin{table}[h]
%\caption{Summary of existing matrix sketching algorithms and complexities.}\vskip -4mm
%\begin{center}
%\begin{tabular}{cllccll}\hline\hline
%{\bf Category}&\multicolumn{1}{l}{\bf Method}  &\multicolumn{2}{c}{\bf Sampling/Projection step}&\multicolumn{2}{c}{\bf Factorization step}&\\
%&&time&space&time&space\\
%\hline
%&Pseudo-skeleton  \cite{skeleton,nystrom}       &&&\\
%Bilateral&Sketch-CUR   \cite{sketchCUR}          &linear &linear&linear&linear\\
%&Two-step (ours)             &&&\\ \hline
%Bilateral &CUR decomposition \cite{cur}&&&\\
%&Adaptive sampling CUR \cite{adaptive}& quad.&quad.& quad. & quad.\\\hline
%Unilateral&Random projection \cite{random}\\
%&LinearSVD (monte carlo) \cite{montecarlo}& quad.&quad.&linear&linear\\\hline\hline
%\end{tabular}
%\end{center}
%\end{table}

\section{Theoretic Analysis of Bilateral Matrix Sketching}
In this section, we present a novel theoretic analysis on bilateral sketching of general, rectangular matrices. %, using the Pseudo-skeleton method (\ref{eq:W}) as an example.
It links the quality of matrix sketching with the encoding powers of bilateral sampling, and inspires a weighted $k$-means procedure as a novel bilateral sampling scheme.

\subsection{A New Error Bound}

In the literature, most error bound analysis for matrix low-rank approximation is of the following theme: a sampling scheme is pre-determined and then probabilistic theoretical guarantees are derived, typically on how many samples should be selected to achieve a desired accuracy \cite{random,montecarlo,cur,fastmc,adaptive}.
In this work, our goal is quite different: \emph{we want to maximally reduce the approximation error given a fixed rate of sampling}. Therefore, our error bound is expressed in terms of the numerical properties of sampling, instead of a specific sampling scheme. Such a heuristic error bound will shed more light on the design of sampling schemes to fully exploit a limited computing resource, thus particularly useful to practitioners.

%Therefore it will shed lig

%our error will be a different, \emph{heuristic} one. In particular, it  Therefore, our analysis is more on how to . %limited sampling rate using limited aly about how to design the sampling scheme to maximally reduce the approximation error.

Given an input matrix $\A \in\r^{m\times n}$, assume $\A = \P\Q^\top$, where $\P\in\r^{m\times r}$ and $\Q\in\r^{n\times r}$ are exact decomposition. Without loss of generality suppose we select $k$ columns $\C = \A_{[:,\Z^c]}$ and $k$ rows $\R = \A_{[\Z^r,:]}$, where $\Z^c$ and $\Z^r$ are sampling indices.  These indices locate representative instances (rows) in $\P$ and $\Q$, denoted by $\ZZ^r = \P_{[\Z^r,:]}$ and $\ZZ^c = \Q_{[\Z^c,:]}$, respectively. In order to study how the bilateral sampling affects matrix sketching result, we adopt a clustered data model. Let the $m$ rows of $\P$ be grouped to $k$ clusters, where the cluster representatives are rows in $\ZZ^r$; similarly, let the $n$ rows in $\Q$ be  grouped into $k$ clusters, where the cluster representatives are rows in $\ZZ^c$.
Let the $s^r(i)$ be the cluster assignment function that maps the $i$-th row in $\P$ to the $s^r(i)$-th row in $\ZZ^r$; similarly, $s^c(i)$ maps the $i$-th row in $\Q$ to the $s^c(i)$-th row in $\ZZ^c$. Then, the errors of reconstructing $\P$ and $\Q$ using the representatives in $\Z^r$ and $\Z^c$ via respective mapping function $s^r(\cdot)$ and $s^c(\cdot)$ can be defined as %We use $e^r$ and $e^c$ to quantify the encoding errors of the mapping functions  $s^r(\cdot)$ and $s^c(\cdot)$ in $\P$ and $\Q$, as
\begin{eqnarray}\label{eq:eer}
e^r = \sum\nolimits_{l=1}^m \|\P_{[l,:]}-\ZZ^r_{[s^r(l),:]}\|^2,\;\; e^c= \sum\nolimits_{l=1}^n \|\Q_{[l,:]}-\ZZ^c_{[s^c(l),:]}\|^2.
\end{eqnarray}
We also define  $\T^r$ and $\T^c$ as the maximum cluster sizes in $\P$ and $\Q$, respectively, as
\begin{eqnarray}\label{eq:T}
\T^r = \max_{1\leq y\leq k}|\{i : s^r(i) = y\}|,\; \T^c = \max_{1\leq y\leq k}|\{i : s^c(i) = y\}|.
\end{eqnarray} %Let $\C$ and $\R$ be the selected columns and rows ($\R$ is transposed here for convenience in proof), with the intersection matrix $\W$.
Given these preliminaries, we can then analyze how the matrix sketching error is associated with the encoding powers of bilateral sampling (in reconstructing the decompositions $\P$ and $\Q$) as follows.

\begin{theorem} Given an input matrix $\A$, and suppose one samples $k$ columns $\C$ and $k$ rows $\R$, with the intersection $\W$. Then we can bound the approximation error  (\ref{eq:W}) as follows.
\begin{eqnarray}
\left\|\A - \C\W^\dagger\R\right\|_F\leq \left(\sqrt{6k\c} \T^{\frac{3}{2}}\right)\cdot
\sqrt{\mathbf{e^r + e^c}} \;\;+\;\; (k\c\T\|\W^\dagger\|_F)\cdot\sqrt{\mathbf{e^ce^r}}
\end{eqnarray}
Here $\text{T} = \max{\left(T^r,T^c\right)}$, $\T^r$ and $\T^c$ are defined in (\ref{eq:T}), $e^r$ and $e^c$ are defined in (\ref{eq:eer}); and $\theta$ is a data dependent constant. Proof of Theorem~1 can be found in Section~1 of supplementary material.
\end{theorem}

\begin{figwindow}[0,r,{\mbox{\includegraphics[width = 5.3cm,height = 4.3cm]{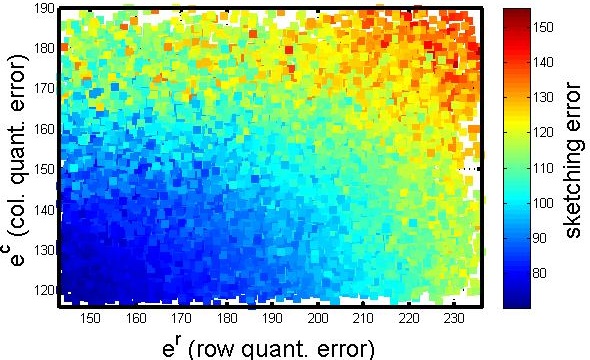}}},{Matrix sketching error vs. row and column encoding errors.}\label{fig:err}]
From Theorem~1, we can see that given a fixed sampling rate $k$ ({\footnotesize so $T$ and $\|\W^\dagger\|_F$ will more or less remain the same too}), the key quantities affecting the matrix sketching error are $e^r$ and $e^c$ (\ref{eq:eer}), the encoding errors of reconstructing $\P$ and $\Q$ with their representative rows whose indices are specified in the bilateral sampling. In case both $e^c $ and $ e^r $ approach zero, the sketching error will also approach zero. Namely, choosing a bilateral sampling that can reduce the encoding errors (\ref{eq:eer}) is an effective way to bound matrix sketching error. We visualized their relations in Figure~\ref{fig:err}. Here, given $\A$ with decomposition $\P\Q^\top$, we perform bilateral samplings many times, each time using an arbitrary choice such as uniform sampling, vector quantization (with different number of iterations), and so on, such that the resultant encoding errors vary a lot. Then we plot the encoding errors $(e^r, e^c)$ and color-code it with the corresponding matrix sketching error $\mathcal{E}$.
As can be seen, $\mathcal{E}$ shows a clear correlation with $e^c$ and $e^r$. Only when both $e^c$ and $e^r$ are small (blue),  $\mathcal{E}$ will be small; or else if either $e^c$ or $e^r$ is large, $\mathcal{E}$ will be large too. %In order to obtain a good approximation, both the row and column encoding errors of bilateral sampling should be small.
\end{figwindow}

\subsection{Sampling on Low-rank Embeddings: Weighted $k$-means}
 Theorem~1 provides an important criterion for sampling: the selected rows and columns of $\A$ should correspond to representative code-vectors in the low-rank embeddings $\P$ and $\Q$, in order for the sketching error to be well bounded. To achieve this goal, we propose to use $k$-means sampling independently on $\P$ and $\Q$, which can quickly reduce their respective encoding errors in just a few iterations. Of course, exact decompositions $\P$ and $\Q$ are impractical and possibly high dimensional. Therefore, we resort to an alternative low-dimensional embedding $\P\in\r^{m\times k},\Q\in\r^{n\times k}$ that will be discussed in detail in the cascaded sampling framework in Section~4.

Here we first study the performance of $k$-means sampling. We note that dense and sparse matrices have different embedding profiles. Energy of dense matrices spreads across rows and columns, so the embedding has a fairly uniform distribution (Figure~\ref{fig:kmeans}(a)). For sparse matrices whose entries are mostly zeros, the embedding collapses towards the origin (Figure~\ref{fig:kmeans}(b). The $k$-means algorithm assigns more centers in densely distributed regions. Therefore the clustering centers are uniform for dense matrices (Figure~\ref{fig:kmeans}(a)), but will be attracted to the origin for sparse matrices (Figure~\ref{fig:kmeans}(b)).
These observations inspire us to perform an importance-weighted $k$-means sampling as follows
\begin{eqnarray}\label{eq:wkmeans}
e^r = \sum\nolimits_{l=1}^m \|\P_{[l,:]}-\ZZ^r_{[s^r(l),:]}\|^2 \cdot\Upsilon(\|\P_{[l,:]}\|_2),\; e^c= \sum\nolimits_{l=1}^n \|\Q_{[l,:]}-\ZZ^c_{[s^c(l),:]}\|^2\cdot\Upsilon(\|\Q_{[l,:]}\|_2).
\end{eqnarray}
Here we use the norm of $\P_{[l,:]}$ (or $\Q_{[l,:]}$) to re-weight the objective of $k$-means in (\ref{eq:eer}), because it is an upper-bound of the energy of the $i$th row in $\A$ (up to a constant scaling), as $\|\A_{[i,:]}\| = \|\P_{[i,:]}\cdot \Q\| \leq \|\P_{[i,:]}\|\cdot \|\Q\|$. The $\Upsilon(\cdot)$ is a monotonous function adjusting the weights (e.g., power, sigmoid, or step function). Here, priority is given to rows/columns with higher energy, and as a result the $k$-means cluster centers will then be pushed away from the origin (Figure~\ref{fig:kmeans}(c)). In practice, we will chose a fast-growing function $\Upsilon$ for sparse matrices, and a slowly-growing (or constant) function $\Upsilon$ for dense matrices, and any $k$-means clustering center will be replaced with its closest in-sample point. Finally, the weighing can be deemed a prior knowledge (preference) on approximating rows and columns of the input matrix, which does not affect the validity of Theorem~1.

\begin{figure}[h]
\begin{center}
\subfigure[dense matrix, k-means sampling]{\psfig{figure=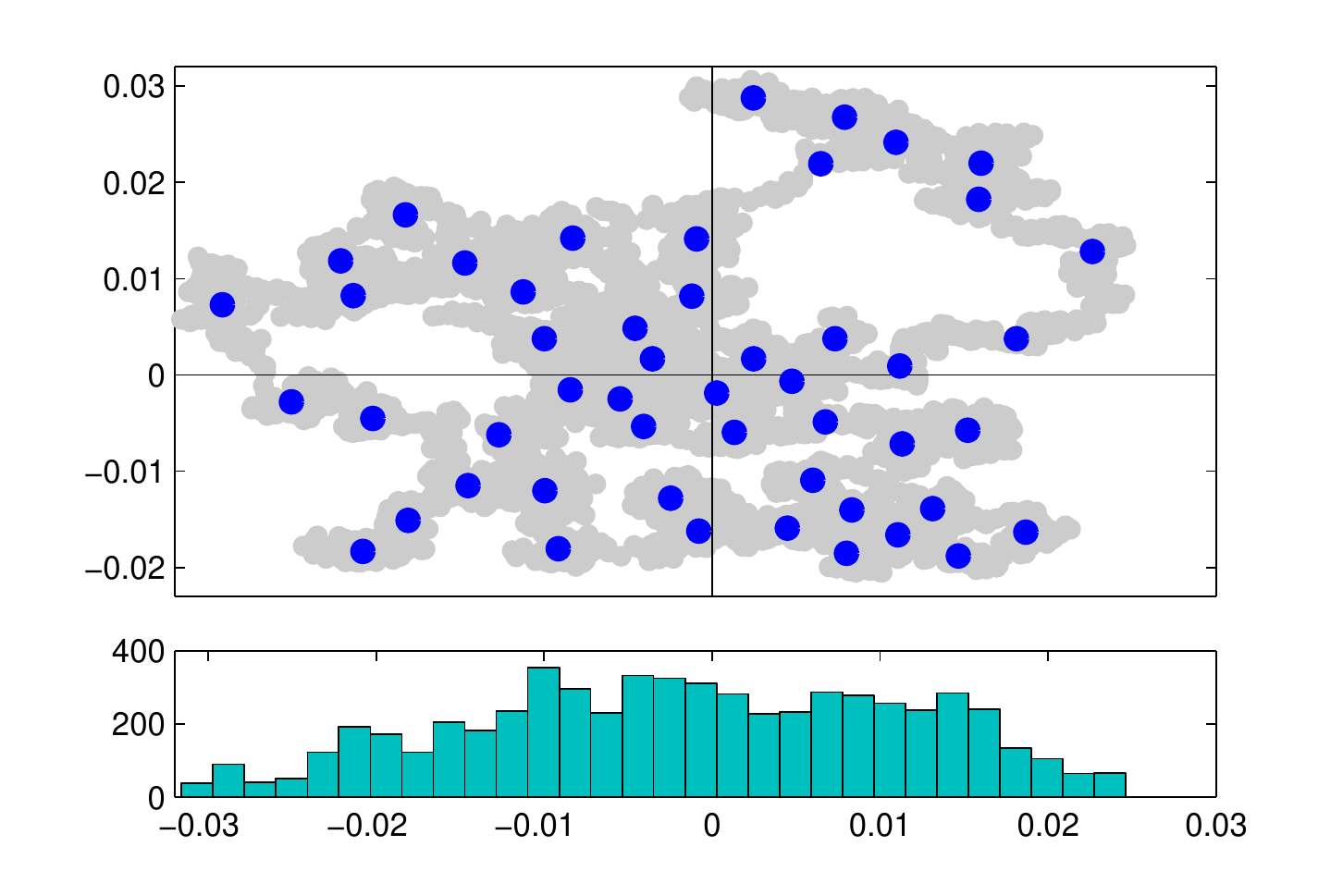,height=4.cm,width=4.7cm}}\hskip -.8mm
\subfigure[sparse matrix, k-means sampling]{\psfig{figure=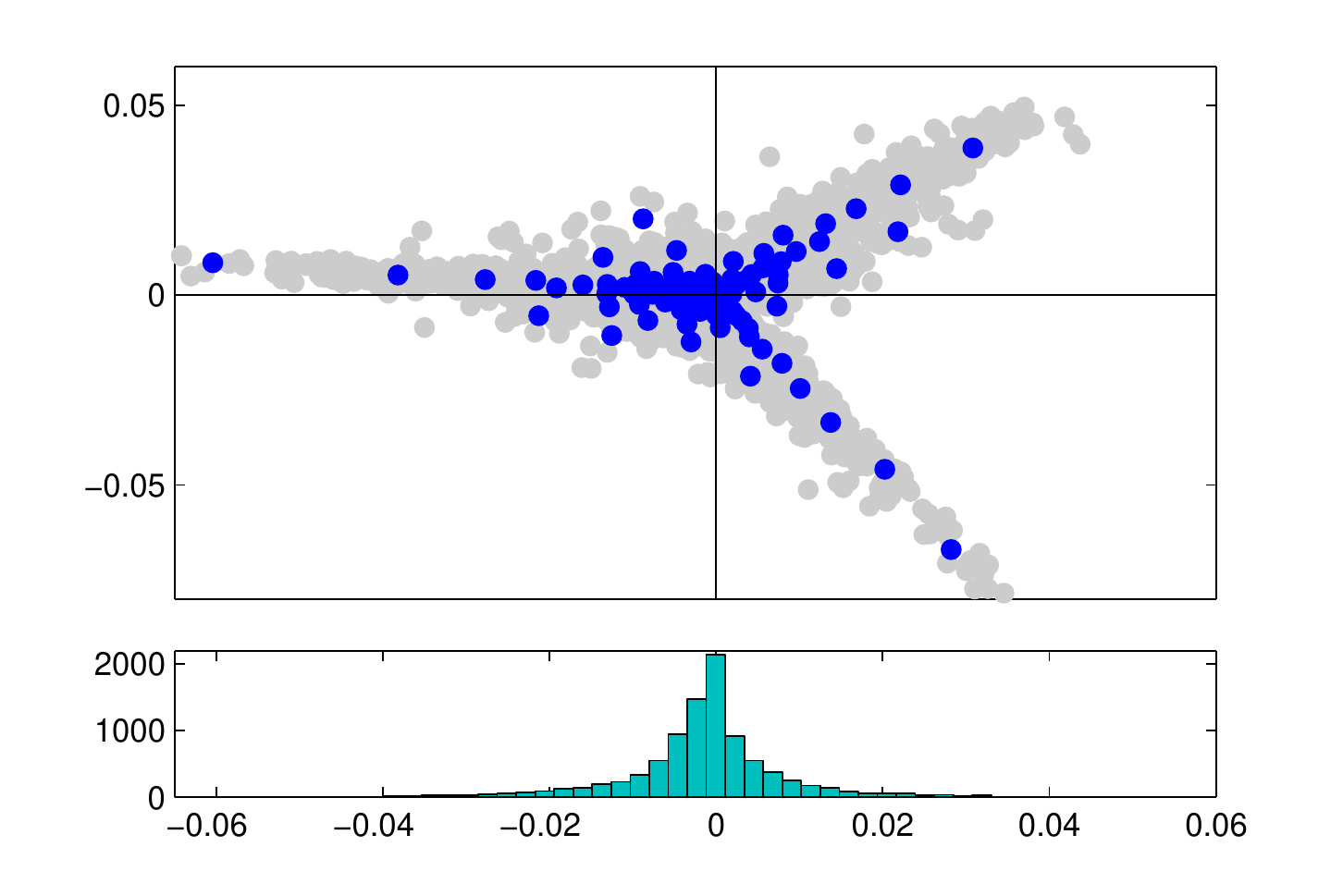,height=4.cm,width=4.7cm}}\hskip -.8mm
\subfigure[sparse matrix, weighted k-means]{\psfig{figure=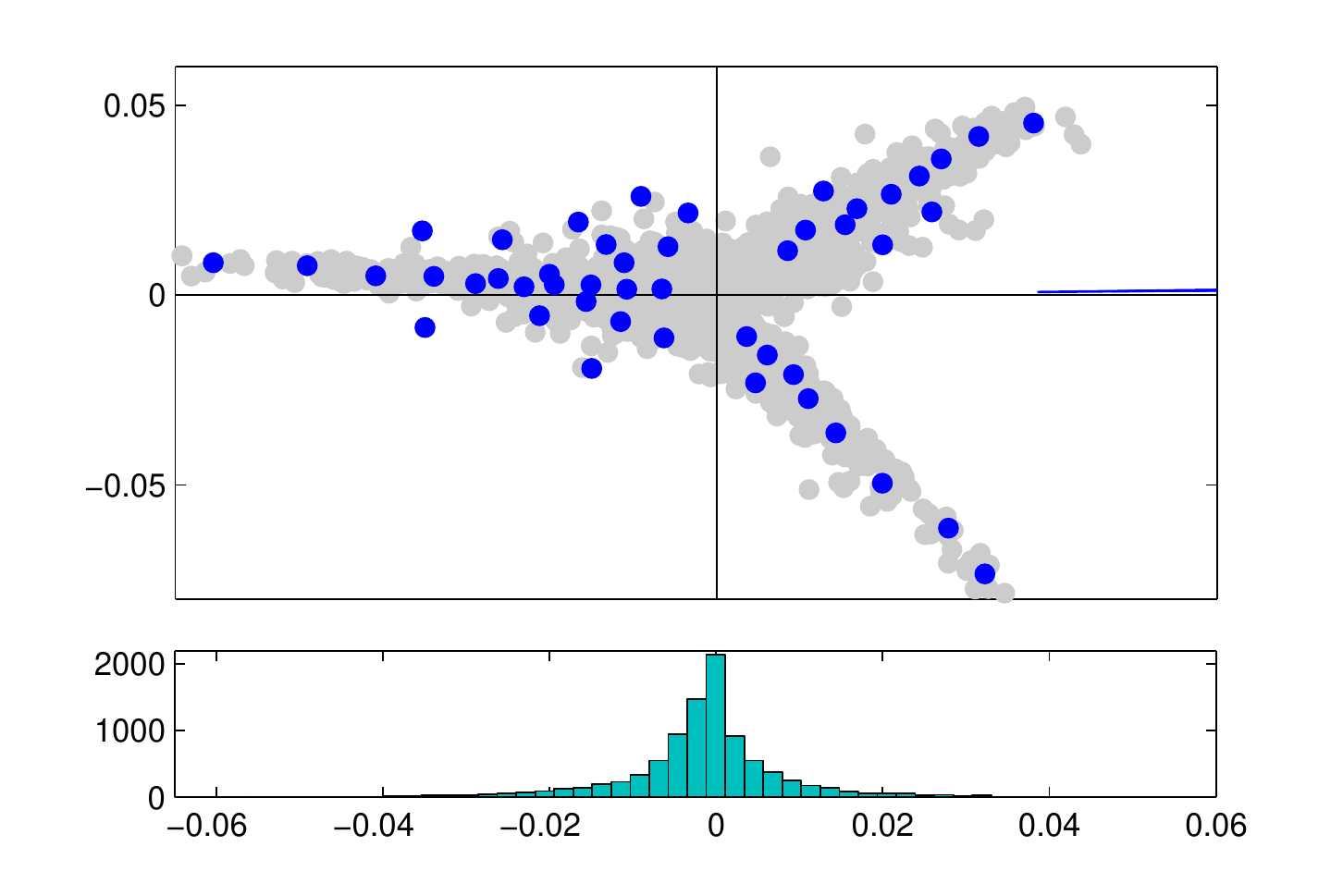,height=4.cm,width=4.7cm}}\vskip -3mm
\caption{\footnotesize Top-2 dimensions of embedding $\P$, histogram on horizontal dimension, and (weighted) $k$-means. }
\label{fig:kmeans}
\end{center}
\end{figure}

%Empirically, for dense matrices we will apply a faster decay function Empirically, the weighting scheme will further improve the sampling in case of sparse input matrices; on dense matrices, the weighting  has little impact on dense matrices but can further improve the results for sparse matrices. %We note that on dense matrices, the weighting scheme has a low impact on the performance; on sparse matrices, however, it can further improve the sketching result.

\section{Cascaded Bilateral Sketching (CABS) Framework}

Theorem~1 suggests that one perform weighted $k$-means sampling (\ref{eq:wkmeans}) on the bilateral embeddings of the input matrix to effectively control the approximation error. Since computing an exact embedding is impractical, we will resort to approximate embeddings discussed in the following framework.%as follows. %in the following cascaded bilateral sampling (CABS) framework.

\begin{algorithm}[H]
\caption{Cascaded Bilateral Sampling (CABS)}
{\bf Input}: $\A$;
    {\bf Output}: $\A\approx\bar{\U}\bar{\mathbf{S}}\bar{\mathbf{V}}^\top$
\label{alg1}

\begin{algorithmic}[1]

\STATE  \emph{{Pilot Sampling}}: {\footnotesize randomly select  $k$ columns and $k$ rows} $\C = \A_{[:,\I^c]}$, $\R = \A_{[\I^r,:]}$, $\W = \A_{[\I^r,\I^c]}$.
\STATE  \emph{{Pilot Sketching}}: run $[\U,\mathbf{S},\mathbf{V}]$ = \texttt{sketching}($\C,\R,\W$), let $\P = \U\mathbf{S}^\frac{1}{2}$, and $\Q = \mathbf{V}\mathbf{S}^\frac{1}{2}$.
\STATE  \emph{{Follow-up sampling}}: perform weighted $k$-means on $\P$ and $\Q$, respectively, to obtain row index $\bar{\I}^r$ and column index $\bar{\I}^c$; let  $\bar{\C} = \A_{[:,\bar{\I}^c]}$, $\bar{\R} = \A_{[\bar{\I}^r,:]}$, and $\bar{\W} = \A_{[\bar{\I}^r,\bar{\I}^c]}$.
\STATE \emph{{Follow-up sketching}}: run $[\bar{\U},\bar{\mathbf{S}},\bar{\mathbf{V}}]$ = \texttt{sketching}($\bar{\C},\bar{\R},\bar{\W}$).
\end{algorithmic}
\end{algorithm}
The CABS framework has two rounds of operations, each round with a sampling and sketching step. In the first round, we perform a simple, random sampling (step~1) and then compute a pilot sketching of the input matrix (step~2). Although this pilot sketching can be less accurate, it provides a compact embedding of the input matrix ($\P$ and $\Q$).
In the follow-up round, as guided by Theorem~1, we then apply weighted $k$-means on $\P$ and $\Q$ to identify representative samples (step~3); resultant sampling is used to compute the final sketching result (step~4). As will be demonstrated both theoretically (Section~4.2) and empirically (Section~5), it is exactly this follow-up sampling that allows us to extract a set of more useful rows and columns, thus significantly boosting the sketching quality.

The \texttt{sketching} routine computes the decomposition of a matrix using only selected rows and columns, as we shall discuss in Section \ref{s:sketching}. As a result, CABS takes only $\mathcal{O}((m+n)(k_1+k_2))$ space and $O((m+n)k_1k_2c)$ time, where $k_1$ and $k_2$ is the pilot and follow-up sampling rate, respectively, and $c$ is the number of $k$-means iterations. In practice, $k_1=k_2\ll m,n$ and $c =5$ in all our experiment, so the complexities are linear in $m+n$. The decomposition is also quite easy to interpret because it  is expressed explicitly with a small subset of representative rows and columns. % the computational advantage, CABS return %Here we assume instant access of any row or column of the input matrix.

\subsection{The \texttt{sketching} routine}
\label{s:sketching}
In this section we discuss the \texttt{sketching} routine used in Algorithm~2. Given a subset of rows $\R$, columns $\C$, and their intersection $\W$ from a matrix $\A$, this routine returns the decomposition $\A\approx\U\S\mathbf{V}^\top$. Both Sketch-CUR and Pseudo-skeleton method are possible candidates, however, in practice they can be numerically sensitive. As shown in Figure~\ref{fig:sensitive}, their approximation error varies with the number of singular vectors used in computing the pseudo-inverse. The optimal number can be small and different from data to data (w.r.t. matrix size). %A detailed examination shows that the optimal tolerance parameter in the pseudo-inverse is much larger than what is commonly used, and varies from data to data.
Another observation is that, Pseudo-skeleton is always superior to Sketch-CUR when they use the same sampling rates.% base and target sampling rates. %rates for base/target samplings).%when choosing the same number of samples in the base and target sampling. %While in common practice we usually use very small tolerance factors.
%\footnote{We verified this by checking  singular value spectrum, actually the optimal tolerance (over $\|\A\|_2$) varies from data to data and is much larger than we expect. Surprisingly,  We are still pursuing theoretic interpretations of this observation.  } . Whereas in common practices, we are so accustomed to using a small tolerance factor (or jittering factor). %This is why empirically previous attempts seldom produced much success.

\begin{figure}[h]
\begin{center}
\subfigure[natural scene]{\psfig{figure=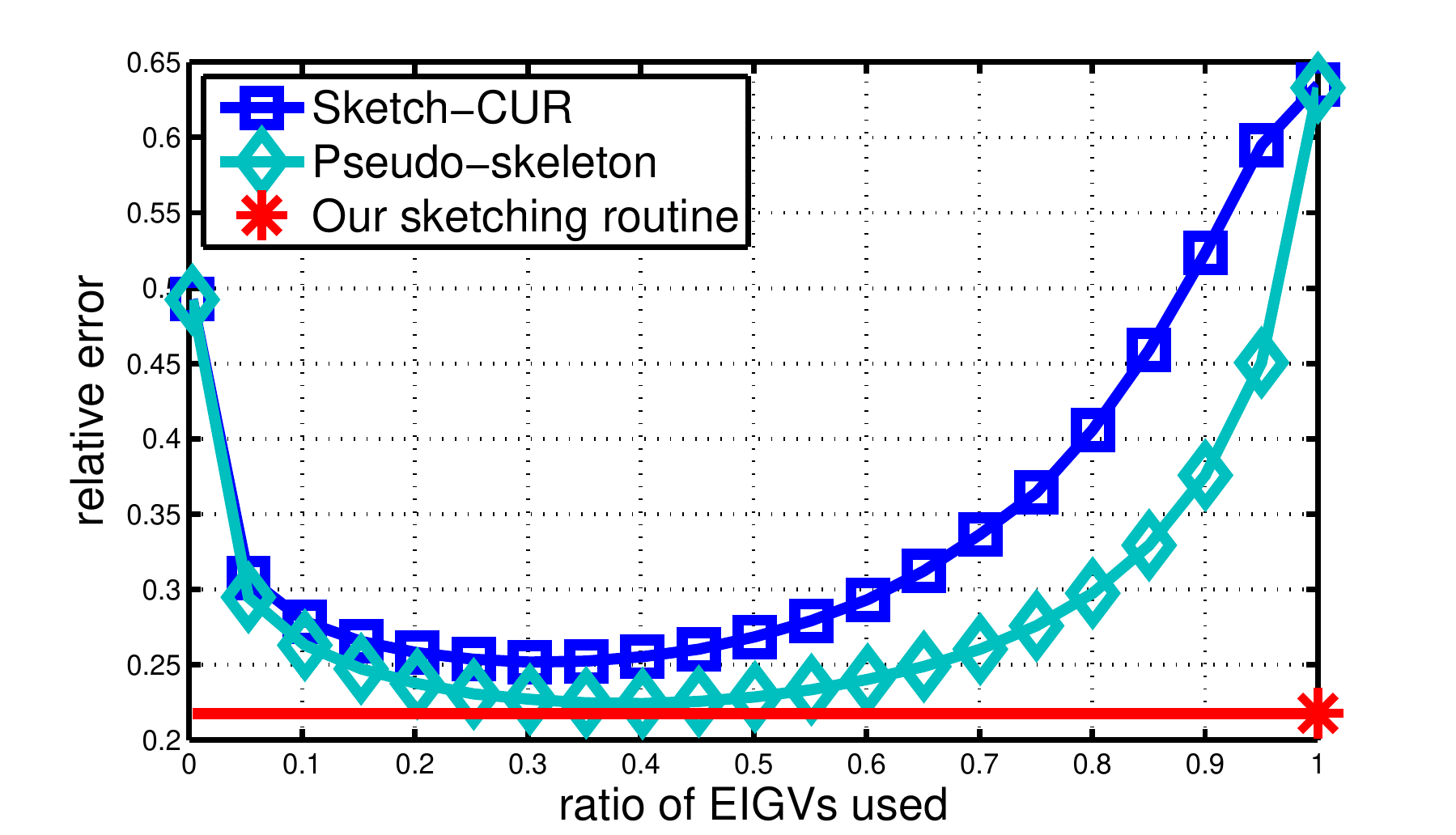,height=3cm,width=4.5cm}}
\subfigure[movie ratings]{\psfig{figure=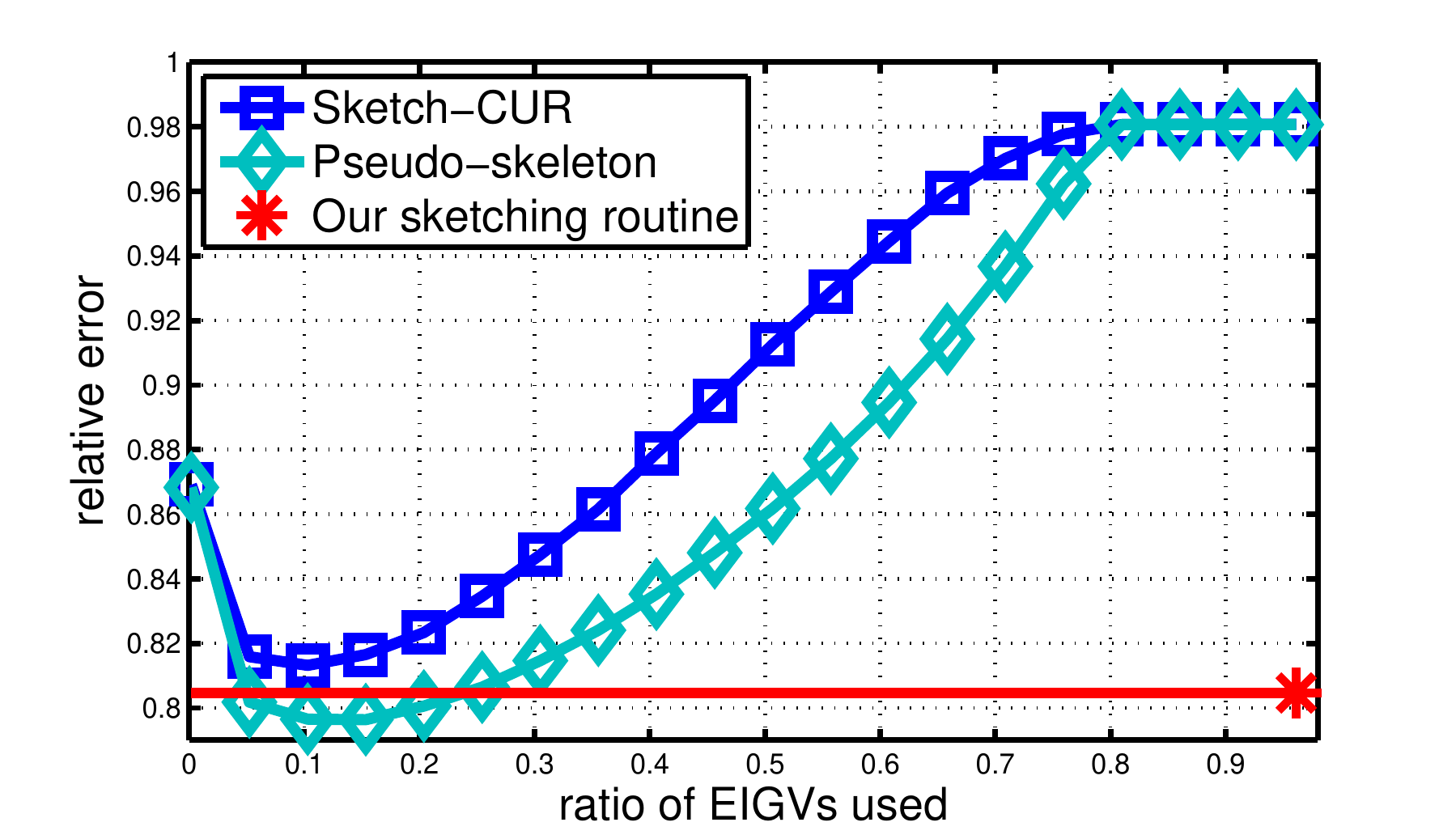,height=3cm,width=4.5cm}}
\subfigure[newsgroup]{\psfig{figure=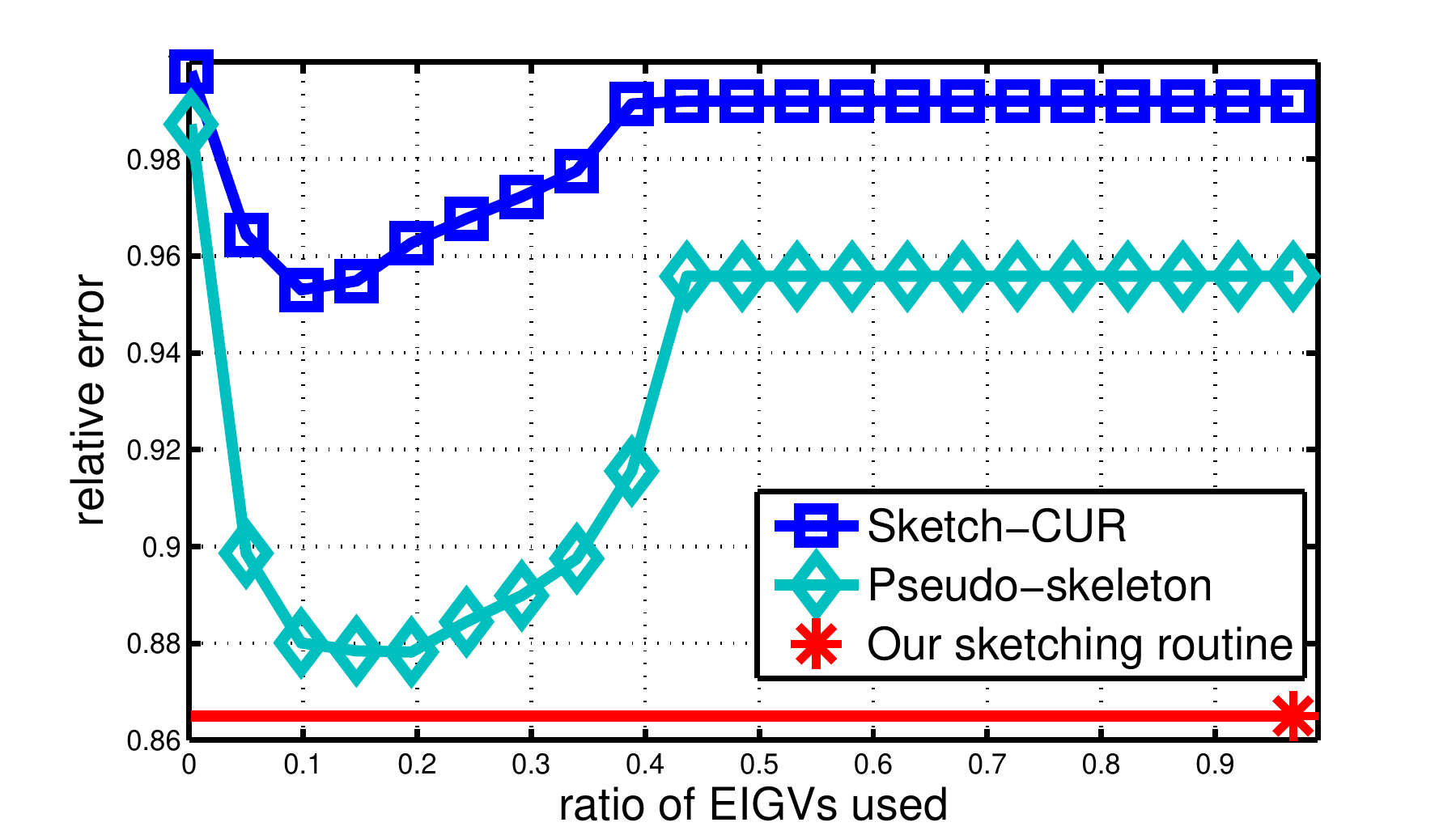,height=3cm,width=4.5cm}}
\vskip -3mm
\caption{Performance of linear-cost algorithms is sensitive to the number of singular vectors used.}
\label{fig:sensitive}
\end{center}\end{figure}

%To choose a rightp tolerance parameter in the pseudo-inverse, one san sample a small block of the %matrix as a validation set. Though feasible, it requires more computations.
In the following we propose a more stabilized variant of Pseudo-skeleton method (Equation \ref{eq:W}).
Assuming an SVD  $\W = \U_w {\SS}_w \V_w^\top$, then $\W^\dagger = \V_w{\SS}_w^{-1}{\SS}_w{\SS}_w^{-1}\U_w^\top$. Plug this into ${\A}\approx \C \W^\dagger \R$, we have
\[
{\A}\approx \left(\C\V_w{\SS}_w^{-1}\right){\SS}_w\left({\SS}_w^{-1}\U_w^\top \R\right).
\]
Here, $\U_w$ and $\V_w$ are left and right singular vectors of $\W$, extrapolated via $\U_w^\top\R$ and $\C\V_w$, respectively, and then normalized by the singular values ${\SS}_w$.
In case ${\SS}_w(i,i)$ approaches zero, the normalization becomes numerically unstable. To avoid this ambiguity, we propose to use the norms of the extrapolated singular-vectors for normalization, as
\begin{eqnarray*}\label{eq:stb}
{\A}\approx \left(\C\V_w\mathbf{N}_c^{-1}\right)\sqrt{\frac{mn}{k^2}}{\SS}_w\left(\mathbf{N}_r^{-1}\U_w^\top \R\right), \;\; s.t.\;\; \mathbf{N}_c = \text{diag}(\|\C\V_w\|_{\otimes}), \; \mathbf{N}_r =\text{diag}(\|\R^\top\U_w\|_\otimes)
\end{eqnarray*}
Here $diag(\cdot)$ fills a diagonal matrix with given vector, $\|\cdot\|_\otimes$ returns column-wise norms, namely $\mathbf{N}_r$ and $\mathbf{N}_c$ are norms of extrapolated singular vectors. The constant $\sqrt{{mn}}/k$ adjusts the scale of solution. We can then define \texttt{sketching} routine with $\U = \C\V_w\mathbf{N}_c^{-1}$, $\mathbf{V} = \mathbf{N}_r^{-1}\U_w^\top \R$, and $\SS = {\SS}_w\sqrt{{mn}}/{k}$. As can be seen from Figure~\ref{fig:sensitive}, it gives stable result by using all singular vectors. The solution can be orthogonalized in linear cost  \cite{random} (see Section~3 in supplementary material). In practice, one can also use a validation set to choose the optimal number of singular vectors.% in the Pseudo-skeleton method, which is slightly more expensive. %. Alternatively, we propose a parameter-free, numerically stable variant of the

One might wonder how previous work tackled the instability issue, however, related  discussion is hard to find. In the literature, expensive computations such as power iteration \cite{godec} or high sampling rate \cite{sketchCUR} were used to improve the performance of Pseudo-skeleton and Sketch-CUR, respectively. Neither suits our need. In comparison, our solution requires no extra computations or sampling.

% these methods, such as using power iterations to improve  method. For example,  used , which involves multiplying the whole input matrix and takes a quadratic cost;  used a target sampling rate several times higher than base sampling to improve .

In positive semi-definite (PSD) matrices, intriguingly, numerical sensitivity diminishes by sampling the same subset of rows and columns, which reduces to the well-known Nystr\"om method \cite{nystrom_williams,nyspami,nystrom} (see Section~2 of supplementary material for detail). Namely, Nystr\"om method is numerically stable on PSD matrices, but its generalized version is more sensitive on rectangular matrices. We speculate that PSD matrix resides in a Riemannian manifold \cite{psd}, so symmetric sampling better captures its structure, however rectangular matrix is not endowed with any structural constraint. This makes the approximation of rectangular matrices particularly challenging compared with that of PSD matrices, and abundant results of the Nystr\"om method may not be borrowed here directly \cite{nystrom_williams,nyspami,nysample,nystrom,mynys}. By stabilizing the \texttt{sketching} routine and employing it in the cascaded sampling framework, CABS will be endowed with superior performance in sketching large rectangular matrices.

\subsection{Algorithmic Boosting Property of CABS}

In the literature, many two-step methods were designed for fast CUR decomposition. They start from an initial decomposition such as fast JL-transform \cite{fastleverage} or random projection \cite{adaptive1,two-step,sketchCUR}, and then compute a sampling probability with it for subsequent CUR \cite{adaptive1,two-step,adaptive}.
We want to emphasize the differences between CABS and existing two-step methods.
\emph{Algorithmically}, CABS only accesses a small part of the input matrix; two-step methods often need to manipulate the whole matrix. Moreover, CABS performs the follow-up sampling on the bilateral low-rank embeddings, which are compact, multivariate summary of the input matrix; in two-step methods, the sampling probability is obtained by reducing the initial decomposition to a univariate probability scores.
\emph{Conceptually}, CABS is targeted on algorithmically boosting the performance of cheap sampling routines. In comparison, two-step methods only focus on theoretic performance guarantees of the second (final) step, and less attention was put on the performance gains from the first step to the second step. % in the second step. % If the second step cannot further improve the performance, it would be quite uneconomical from practical computational point of view. Our preliminary examination demonstrates little such improvement, unfortunately. %viewpoint, if an initial approximation is already accuralthough the initial approximation first step little effort has been devoted on whet
% boosting, namely to achieve a follow-up sketching that is more accurate than the pilot sketching. %can generate much more accurate result than the pilot sketching, which even competes favorably with state-of-the-art, quadratic-cost algorithms. In other words, the cascading mechanism achieves an algorithmic boosting.

\emph{How to quantify the rise (or drop) of approximation quality in a two-step method?
How to choose the right pilot and follow-up sampling to save computational costs and maximize performance gains?} These are fundamental questions we aim to answer.
For the second question, we address it by creatively cascading random sampling with a weighted $k$-means, thus deriving a working example of ``algorithmic boosting''. One might want to try pairing existing sketching/sampling algorithms to achieve the same goal, however, choosing an effective follow-up sampling from existing strategies can be non-trivial (see Section~5 for detailed discussion). %yet the performance degenerates on dense matrices (Section~5). %So the success of CABS is not only due to stabilized {sketching}, but also the weighted $k$-means sampling which is the key to the algorithmic boosting. %is an important  contribution.
For the first question, a thorough theoretic answer is quite challenging, nevertheless, we still provide a useful initial result. We show that, the decrement of the error bound from the pilot sketching to the follow-up sketching in CABS, is lower bounded by the drop of  encoding errors achieved through the follow-up sampling of $k$-means. In other words, the better the encoding in the follow-up sampling step, the larger the performance gain.

\begin{theorem}
Let the error bound of the pilot and follow-up sketching in CABS be $\Psi_p$ and $\Psi_f$, respectively. Then the error bound will drop by at least the following amount
\begin{eqnarray*}
\Psi_p - \Psi_f \geq
 \sqrt{\frac{3k\theta T^c_pT^r_p(T^r_p+T^c_p)}{2(e^r_p + e^c_p)}}\mathbf{\left|\Delta^r_e + \Delta^c_e\right|} + k\theta\left\|\W_p^\dagger\right\|_F\sqrt{\frac{{T^cT^r}}{e^r_p e^c_p}}\left|\mathbf{\Delta^r_e} e^c_f + \mathbf{\Delta^c_e} e^r_f + \mathbf{ \Delta^r_e\Delta^c_e}\right|.
\end{eqnarray*}
Parameters are defined in the same way as in Theorem~1, and sub-index $\{p,f\}$ denotes pilot or follow-up; $\mathbf{\Delta^r_e}$ ($\mathbf{\Delta^c_e}$) is the drop of row (column) encoding error in the follow-up $k$-means sampling. Proof of Theorem~2 can be found in Section~4 of supplementary material.
\end{theorem}

The algorithmic boosting effect is not dependent on the \texttt{sketching} routine adopted in Algorithm~2. Empirically, by using less stable routines (Pseudo-skeleton or Sketch-CUR), significant performance gains are still observed. Namely CABS is a general framework to achieve algorithmic boosting. In practice, the superior performance of CABS is attributed to both reliable \texttt{sketching} routine and the cascading/boosting mechanism.%, with the latter playing a probably more important role. % (the latter can be more significant). % contributes to the superior performance of CABS, and the latter can be more significant.

 % far beyond the scope of this paper. Currently, we are pursuing more rigorous theoretic delineations of the CABS algorithm behaviour to answer these questions.

%With the stabilized \texttt{sketching} routine, one can obtain a more reliable pilot-sketching, thus contributing to a good follow-up-sketching. However, we want to emphasize that, the success of the CABS framework should as well be attributed to the \emph{cascading} mechanism. As we shall see from the empirical results (Figure~8), after cascading a second round of sampling, the sketching performance can be significantly improved by a very large margin.

\section{Experiments}

All experiments run on a server with 2.6GHZ processor and 64G RAM. Benchmark data sets are described in Table~1. %For two of these matrices, we used sub-sampling to reduce their sizes. Although our approach only has a linear memory cost, many competing algorithms need to manipulate the entire matrix.
All codes are written in matlab and fully optimized by vectorized operations and matrix computations, to guarantee a fair comparison in case time consumption is considered.
\vskip -5mm \begin{footnotesize}
\begin{table}[h]
\caption{Summary of benchmark data sets.}
\label{sample-table}
\begin{center}
\begin{tabular}{llllll}
data& \#row        &\#column &sparsity&source urls\\
\hline\hline
movie ratings & 27000&46000 &0.56\%& {\sf \footnotesize http://grouplens.org/datasets/movielens/} \\
newsgroup&18774   &61188& 0.22\%& {\sf\footnotesize http://kdd.ics.uci.edu/databases/20newsgroups/}\\
natural scene & 6480&7680 & dense & {\sf\footnotesize http://wallpapershome.com/}\\
{hubble image} & 15852&12392& dense & {\sf\footnotesize http://hubblesite.org/gallery/}\\
\hline
\end{tabular}

\end{center}
\end{table}
\end{footnotesize}

First we compare all linear-cost algorithms, including: (a) Sketch-CUR \cite{sketchCUR}, where the target sampling rate is chosen three times as much as the base sampling rate; (b) Pseudo-skeleton \cite{skeleton}, which is the generalized Nystr\"om method; (c) Pilot-sketch, which is step~2 of Algorithm~2; (d) Followup-sketch (w-kmeans), which is step~4 of Algorithm~2; (e) Followup-sketch (leverage), a variant of our approach using approximate leverage scores for the follow-up sampling; (f) Followup-sketch (hard-thrhd), a variant of our approach using top-$k$ samples with largest weighting coefficients (Equation \ref{eq:wkmeans}) for the follow-up sampling.
In Figure~\ref{fig:compariosn}(a), we gradually increase the number of selected rows and columns from 1\% to 10\%, and then report the averaged error $\|\A - \tilde{\A}\|_F/\|\A\|_F$ over 20 repeats.
For simplicity, the number of selected rows and columns are both $k$, and the sampling rate is defined as $k/\sqrt{mn}$.
Our observations are as follows: (1) our pilot sketching is already more accurate than both Pseudo-skeleton and Sketch-CUR; (2) our follow-up sketching result using the weighted $k$-means sampling strategy is consistently and significantly more accurate than our pilot sketching, which clearly demonstrates the ``algorithmic boosting'' effect of our framework; (3) on using leverage scores for the follow-up sampling, the performance gain is lower than the weighted $k$-means strategy, and becomes insignificant on dense matrices; we speculate that on dense matrices, since leverage scores are more uniformly distributed (Figure~\ref{fig:kmeans}(a)), they are less discriminative and can introduce redundancy when used as sampling probabilities; (4) using hard-threshold sampling on the weighting coefficients is particularly beneficial on sparse matrices, but the results degenerate on dense matrices and so are skipped. This is because the weighting coefficients are clear indicators of the importance (or energy) of the rows or columns, and can be very discriminative on sparse matrices (see Figure~\ref{fig:kmeans}(b)); on dense matrices, all the rows and columns can be equally important, therefore the weighting coefficients are no longer informative.% -distribution . For the last two observatios
\begin{figure}[h]
\begin{center}

\subfigure[Comparison with linear-cost (BR-CUR) algorithms]{\psfig{figure=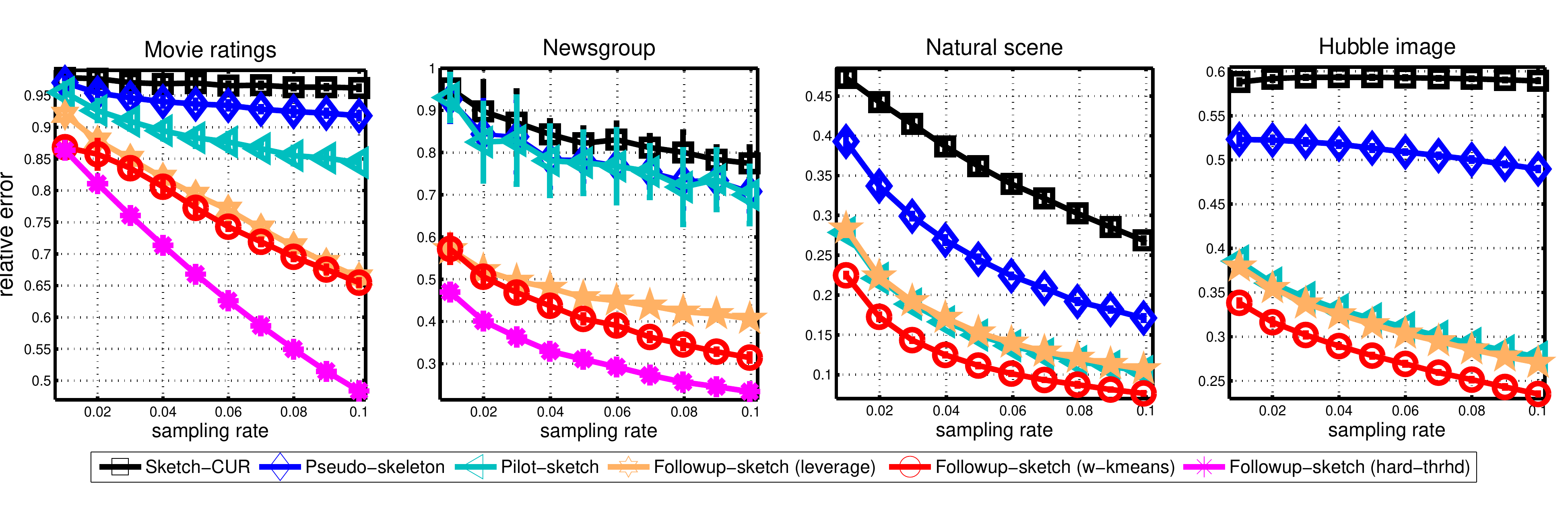,height=3.6cm,width=14cm}}
\vskip -2mm
\subfigure[Comparison with quadratic-cost (randomized) algorithms]{\psfig{figure=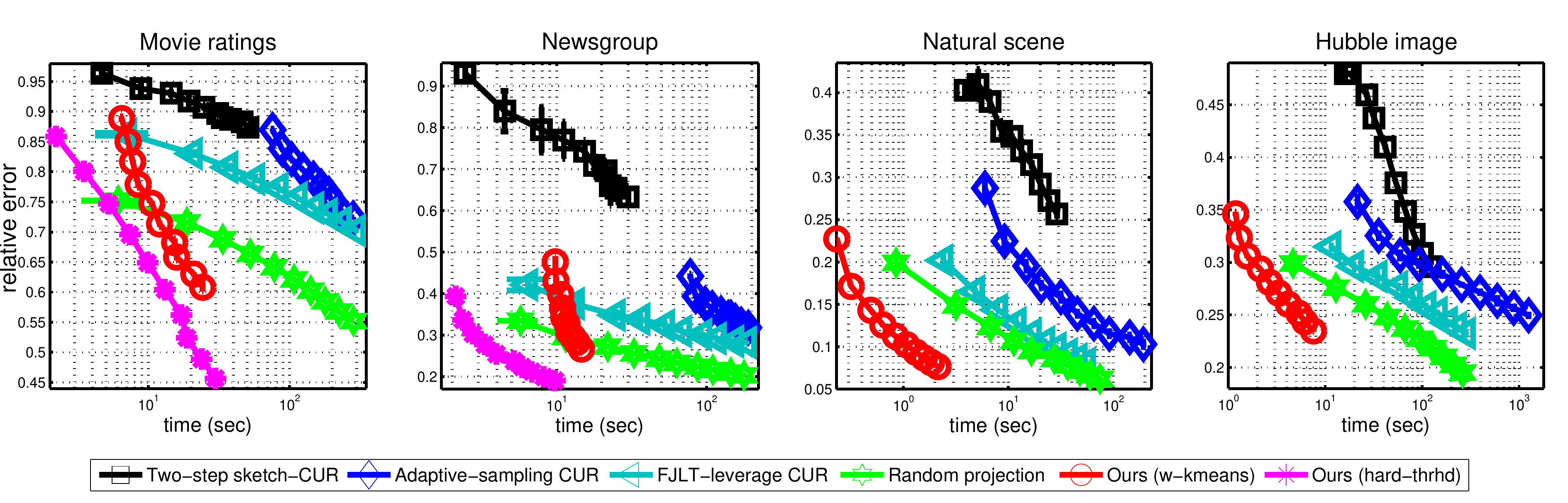,height=3.6cm,width=14cm}}
%\hskip-5mm\includegraphics[width = 16cm,height = 4.5cm]{r1fig.eps}
\vskip -3mm
\caption{Sketching performance of our method and competing methods.}\label{fig:compariosn}

\end{center}
\label{fig:comp1}
\end{figure}
%\begin{figure}[h]
%\begin{center}
%\hskip-5mm\includegraphics[width = 16cm,height = 4.65cm]{r2fig.eps}
%\caption{Reconstruction error of our approach and quadratic-cost algorithms.}\label{fig:cur}\vskip -2.5mm
%\end{center}
%\label{fig:comp1}
%\end{figure}

Then we compare our approach with state-of-the-art randomized algorithms (all with quadratic costs), including (a) Two-step Sketch-CUR, which first performs random projection \cite{random} and then Sketch-CUR \cite{two-step}; (b) Adaptive sampling CUR \cite{adaptive}, which uses the error distribution of an initial approximation to guide extra sampling; (c) FJLT-leverage CUR, which uses fast JL-transform for initial decomposition and then perform CUR \cite{fastleverage}; (d) Random projection \cite{random}, with $q=1$ step of power iteration; (e) Ours (w-kmeans); and (f) Ours (hard-thrhd) for sparse matrices. Each algorithm is repeated 20 times and averaged approximation error over time consumption is reported in Figure~\ref{fig:compariosn}(b).
We have the following observations: (1) Random projection is very accurate but the time consumption can be significant; (2) FJLT-leverage CUR has similar time consumption and can be less accurate; (3) Adaptive-sampling CUR is computationally most expensive (in computing the residue of approximation); (4) Two-step Sketch-CUR is the least accurate, which we speculate is due to the instability of Sketch-CUR; (5) Our approach with weighted $k$-means sampling has a clear computational gain in dense matrices with a good accuracy; (6) Our approach using hard-threshold sampling performs particularly well in sparse matrices. Overall, our approaches have competing accuracies but significantly lower time and space consumption. The larger the input matrices, the higher the performance gains that can be expected. % (around one order of magnitude).

\section{Conclusion and Future Work}

In this paper, we propose a novel computational framework to boost the performance of cheap sampling and sketching routines by creatively cascading them together. Our methods is particularly time and memory efficient, and delivers promising accuracy that matches with state-of-the-art randomized algorithms. In the future, we will pursue more general theoretic guarantees and delineations; we are also applying the framework in parallel environment with some promising preliminary results.%, see appendix for initial results. %, and apply it in more challenging problems of big data analytics.
\vspace{-0.1in}
\begin{footnotesize}
%\bibliographystyle{plain}
%\bibliography{nipsref}

\end{footnotesize}

\section{Proof of Theorem 1}
In order to prove theorem 1, we need to partition the input matrix $\A$ into small, equal-sized  blocks. Note that this new partition will be determined based on the clusters defined on Section~3.1. Remind that the rows in $\P$ are grouped into $k$ clusters with cluster size $T^r_{(i)}$. We add virtual instances to $\P$ such that all clusters have the same size, i.e., $T^r = \max T^r_{(i)}$. The virtual instances added to the $q$th cluster are chosen as the $q$th representative in $\ZZ^r$, therefore they induce no extra quantization errors. Then, we can re-partition $\P$ into $T^r$ partitions each containing exactly $k$ instances. We use $\I^r_i$ to denote the indexes of these partitions for $i = 1,2,..., \T^r$; similarly rows in $\Q$ are also completed by virtual rows, falling in partitions $\I^c_j$ for $j = 1,2,..., \T^c$. Equivalently, $\A$ is augmented into $k\T^r$-by-$k\T^c$ matrix, forming a number of $\T^c\T^r$ blocks each is of size $k$-by-$k$. The approximation error on each of these blocks will be quantized as follows.

\begin{prop}
Let the input matrix $\A$ be augmented as described above, and the resultant decompositions $\P$ and $\Q$ are re-organized into $\T^r$ and $T^c$ equal-sized groups, indexed by $\I^r_i$ for $i = 1,2,..., T^r$, and $\I^c_j$ for $j = 1,2,..., T^c$. Then the approximation on each block defined by indexes $\I^r_i$ and $\I^r_i$ is as follows.
\begin{eqnarray}
\frac{1}{\sqrt{k\theta}}\left\|\A_{[\I^r_i,\I^c_j]} - \C_{[\I^r_i,:]}\W^\dagger\R_{[\I^c_j,:]}^\top\right\|_F
\leq \sqrt{e^r_j + e^c_j} + \sqrt{e^r_i} + \sqrt{e^c_i}  + \sqrt{k\c e^r_ie^c_j}\left\|W^\dagger\right\|_F.
\end{eqnarray}
Here $e^r_i = \sum_{l\in \I^r_{i}}\|\P_{[l,:]}-\ZZ^r_{[s(l),:]}\|^2$ is the error of encoding rows of $\P$ (specified by $\I^r_i$) with representative $\ZZ^r$ via the mapping $s^r$.  Similarly, $e^c_i = \sum_{l\in \I^c_{i}}\|\P_{[l,:]}-\ZZ^c_{[s(l),:]}\|^2$ is the encoding error of encoding rows in $\Q$ (specified by $\I^c_j$) with representative $\ZZ^c$.
\end{prop}

We first establish some basic equalities to use in the proof. Define $\langle\X,\Y\rangle = \X\Y^\top$ for matrices $\X,\Y$ with proper dimensions.
\begin{eqnarray*}
\C_{[\I^r_i,:]} &=& \left\langle\P_{[\I^r_i,:]},\Q_{[\Z^c,:]}\right\rangle,\\
\R_{[\I^c_j,:]}& =& \left\langle\P_{[\I^c_j,:]},\Q_{[\Z^r,:]}\right\rangle,\\
\A_{[\I^r_i,\I^r_j]} &=& \left\langle\P_{[\I^r_i,:]},\Q_{[\I^c_j,:]}\right\rangle,\\
\W &=& \left\langle\P_{[\Z^r,:]},\Q_{[\Z^c,:]}\right\rangle.
\end{eqnarray*}
Here we have used the transposed version of $\R$ for convenience of proof. In other words $\R$ will be an $n\times k$ matrix, which is the transpose of its counterpart in theorem~1 or the CUR decomposition. The change of representation won't affect the correctness of our proofs. %, in $\R$  comparison to Theorem~1) for convenience in proving the theorem.

We also define the following difference matrices
\begin{eqnarray*}
\Delta_C &=& \C_{[\I^r_i,:]} - \W,\\
\Delta_R &=& \R_{[\I^c_j,:]}  - \W,\\
\Delta_A &=& \A_{[\I^r_i,\I^c_j]}  - \W,
\end{eqnarray*}

\begin{prop}Let $\x_1,\x_2,\y_1,\y_2$ be $1\times d$ vectors (to be consistent with our definitions), and let $\f(\cdot,\cdot)$ be the inner product between two such vectors, then we have the following inequality
\begin{eqnarray}
(\f\left( \x_1,\y_1\right) - \f\left( \x_2,\y_2\right))^2 \leq \theta \cdot(\|\x_1 - \x_2\|^2 + \|\y_1 - \y_2\|^2).
\end{eqnarray}
\end{prop}
\begin{proof}
Using the Lagrangian mean-value theorem, we can have
\begin{eqnarray*}
(\f\left(\x_1,\y_1\right) - \f\left( \x_2,\y_2\right)^2 &=& (\f'(\xi)(\x_1\y_1' - \x_2\y_2'))^2\\
& = & \f'(\xi)^2(\x_1\y_1' - \x_1\y_2' + \x_1\y_2' - \x_2\y_2')^2\\
& = & \f'(\xi)^2\left(\x_1(\y_1 - \y_2)' + (\x_1 - \x_2)\y_2'\right)^2\\
&\leq & 2\f'(\xi)^2 \left((\x_1(\y_1 - \y_2)')^2 + (\x_1 - \x_2)\y_2'^2\right)\\
& \leq &\theta \left( \|\x_1 - \x_2\|^2 + \|\y_1 - \y_2\|^2 \right)
\end{eqnarray*}
where $\theta = 2f'(\xi)^2$. This completes the proof of Proposition~2.\end{proof}
Using Proposition~2, we can bound the norms of the difference matrices as follows,
\begin{eqnarray}
\nonumber \|\Delta_A\|_F^2 &=& \left\|\A_{[\I^r_i,\I^c_j]} - \W\right\|_F^2\\
\nonumber&=& \left\|\left\langle\P_{[\I^r_i,:]},\Q_{[\I^c_j,:]}\right\rangle - \left\langle\P_{[\Z^r,:]},\Q_{[\Z^c,:]}\right\rangle\right\|_F^2\\
\nonumber &=& \sum_{p,q = 1}^k \left[\f\left(\P_{[\I^r_i(p),:]},\Q_{[\I^c_j(q),:]}\right) - \f\left(\P_{[\Z^r(p),:]},\Q_{[\Z^c(q),:]}\right)\right]^2\\
\nonumber & \leq &\theta\sum_{p,q=1}^{k}\left(\left\|\P_{[\I^r_i(p),:]} - \P_{[\Z^r(p),:]}\right\|^2 + \left\|\Q_{[\I^c_j(q),:]} - \Q_{[\Z^c(q),:]}\right\|^2\right)\\
\nonumber & = & k\theta\left(\sum_{p = 1}^k \left\|\P_{[\I^r_i(p),:]} - \P_{[\Z^r(p),:]}\right\|^2 + \sum_{q = 1}^k\left\|\Q_{[\I^c_j(q),:]} - \Q_{[\Z^c(q),:]}\right\|^2 \right)\\
& = & k\theta\left(e^r_i + e^c_j\right).\label{eq:A}
\end{eqnarray}
Here we have used the pre-defined relation $s^r(\I^r_i(p)) = p$, and $s^c(\I^c_j(q)) = q$, since the partition index $\I^c_i$ and $\I^c_j$ has the corresponding representative set $\Z$.

Similarly, we have
\begin{eqnarray}
\nonumber \|\Delta_C\|_F^2 &=& \left\|\C_{[\I^r_i,:]}  - \W\right\|_F^2\\
\nonumber & = & \left\|\left\langle\P_{[\I^r_i,:]},\Q_{[\Z^c,:]}\right\rangle - \left\langle\P_{[\Z^r,:]},\Q_{[\Z^c,:]}\right\rangle\right\|_F^2\\
\nonumber & \leq &\theta\sum_{p,q = 1}^k \left(\left\|\P_{[\I^r_i(p),:]} - \P_{[\Z^r(p),:]}\right\|^2 \right)\\
& = & \theta k e^r_i. \label{eq:C}
\end{eqnarray}
and
\begin{eqnarray}\label{eq:R}
\|\Delta_R\|_F^2 \leq \theta k e^c_j.
\end{eqnarray}

By using (\ref{eq:A}), (\ref{eq:C}), and (\ref{eq:R}), we have
\begin{eqnarray*}
\left\|\A_{[\I^r_i,\I^c_j]} - \C_{[\I^r_i,:]}\W^\dagger\R_{[\I^c_j,:]}^\top\right\|_F &=& \|(\W + \Delta_A) - (\W + \Delta_R)\W^\dagger (\W + \Delta_C)\|\\
& = & \|\Delta_A -\Delta_R - \Delta_C - \Delta_R\W^\dagger\Delta_C\|_F\\
&\leq & \|\Delta_A\|_F + \|\Delta_C\|_F + \|\Delta_A\|_F\|\Delta_C\|_F\|\W^\dagger\|_F\\
& =& \sqrt{k\theta}\left(\sqrt{e^r_j + e^c_j} + \sqrt{e^r_i} + \sqrt{e^c_i}  + \sqrt{k\c e^r_ie^c_j}\left\|W^\dagger\right\|_F\right).
\end{eqnarray*}
This completes the proof of Proposition~1.

With this proposition, and by using the inequality $\sum_{i}^n\sqrt{x_i}\leq \sqrt{n\sum_{i}x_i}$, the overall approximation error can be bounded as follows
\begin{eqnarray*}
&&\frac{1}{\sqrt{k\c}}\left\|\A - \C\W^\dagger\R^\top\right\|_F \leq \frac{1}{\sqrt{k\c}}\sum_{i=1}^{T^r}\sum_{j=1}^{T^c}\left\|\A_{[\I^r_i,\I^c_j]} - \C_{[\I^r_i,:]}\U^\dagger\R_p{[\I^c_j,:]}^\top\right\|_F\\
&\leq&\sqrt{T^c}\sum_i\sqrt{\sum\nolimits_j(e^c_j + e^r_j)} + \sqrt{T^r}\sum_j \sqrt{\sum\nolimits_i e^r_i} +\sqrt{T^c}\sum_i \sqrt{\sum\nolimits_j e^c_j} \\& &+ \sqrt{k\c}\left\|W^\dagger\right\|_F\sqrt{T^c}\sum_i \sqrt{\sum\nolimits_j e^r_ie^c_j}\\
&\leq& \sqrt{\T^c\T^r(\T^r e^c + T^ce^r)} + T^c\sqrt{T^re^r}+T^r\sqrt{T^ce^c} + \sqrt{k\c}\sqrt{T^cT^re^ce^r}\left\|\W^\dagger\right\|_F\\
&\leq&\sqrt{3(T^c+T^r)(e^c + e^r)} +\sqrt{k\c}\sqrt{T^cT^re^ce^r}\|\W^\dagger\|_F.
\end{eqnarray*} By using $\T = \max{(T^r,T^c)}$, we can easily prove Theorem~1.

\section{Stability Issue}

In this section, we will provide a detailed example showing that: (1) Nystr\"om method is stable on PSD matrices by selecting the same subset of rows and columns in the approximation process; (2) both the pseudo-skeleton method and the sketch-CUR method are unstable on PSD matrices if they perform sampling of the rows and columns independently\footnote{If they sample the same set of rows and columns, then they will be reduced to the Nystr\"om method.}. This observation shows that the symmetric structure of PSD matrices allows one to impose useful constraints on the sampling, so as to guarantee the stability of low-rank approximation. However, in general rectangular matrices, no such constraints are available, therefore the low-rank approximation can be much more difficult.

In Figure~1, we plot the approximation error of the three methods, namely sketch-CUR, pseudo-skeleton, and Nystr\"om method versus the number (ratio) of singular vectors used in computing the pseudo-inverse. The input matrix is chosen as an RBF kernel matrix, which is known to be symmetric, positive semi-definite (PSD). As can be seen, both the sketch-CUR method and the Pseudo-skeleton method are unstable, in that their approximation error is quite sensitive to the number of singular vectors adopted. In comparison, the Nystr\"om method is stable: the more singular vectors adopted in computing the pseudo-inverse, the better the performance.

\begin{center}
\begin{minipage}{1\textwidth}
\centering
\includegraphics[width = 8cm,height = 5cm]{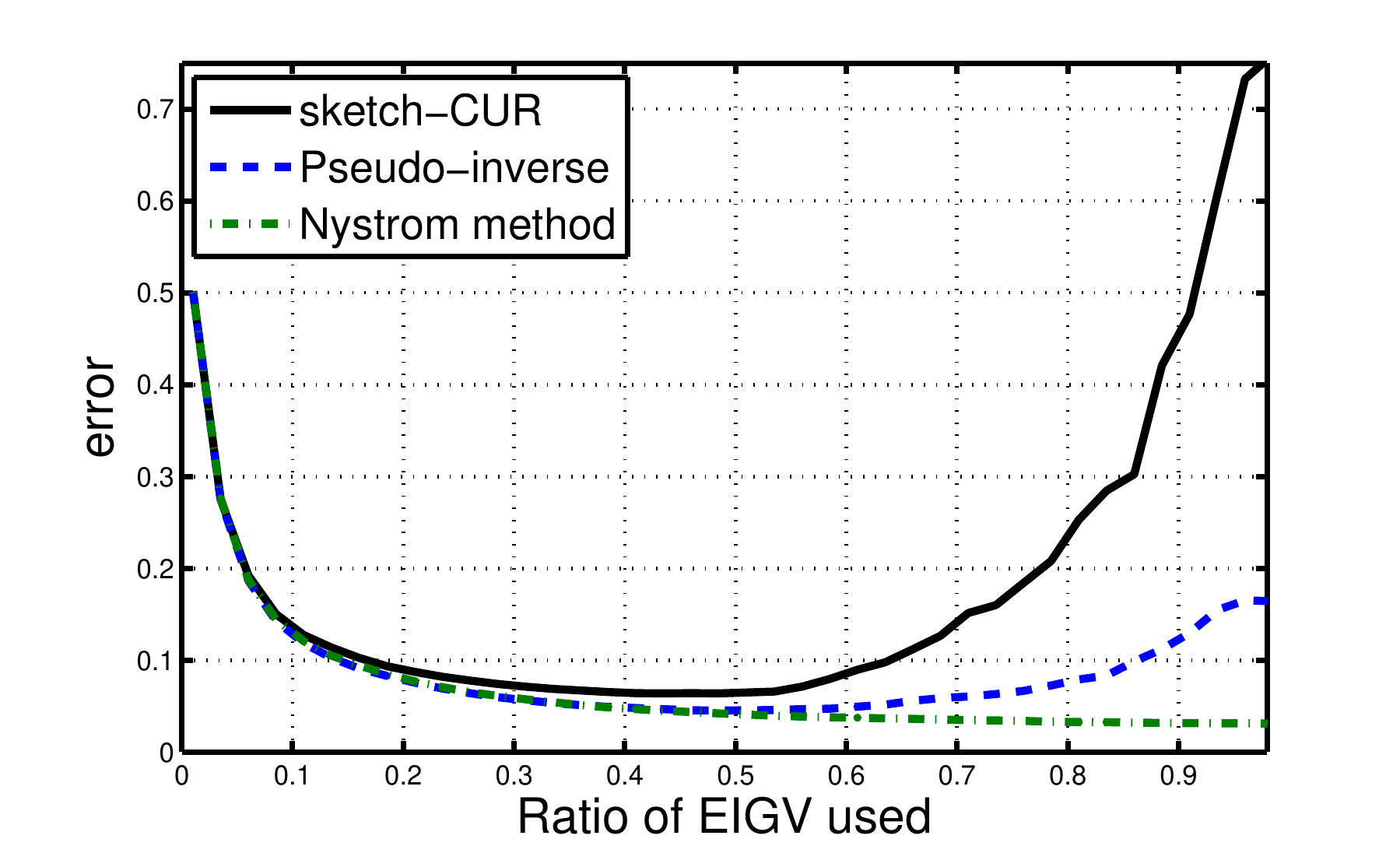}\label{fig:BR-CUR}\vskip -3mm
\captionof{figure}{Approximation error versus the number of singular vectors used.}
\end{minipage}
\end{center}

\section{Orthogonalization Step of \texttt{sketching} Routine}
The \texttt{sketching} routine in CABS (Algorithm~2) gives the following approximation
\begin{eqnarray}
\A \approx \U\SS\mathbf{V}^\top.
\end{eqnarray}
Here $\U\in\mathbb{R}^{m\times r}$,
$\SS\in\mathbb{R}^{k\times k}$ is a diagonal matrix, and
$\mathbf{V}\in\mathbb{R}^{n\times k}$. In using this sketching results for subsequent steps, one has the option to further orthogonalize $\U$ and $\mathbf{V}$, for example, when approximate leverage scores need to be computed. Empirically, we find that the orthogonalization step does not make a significant difference on the final performance, since $\U$ and $\mathbf{V}$  are already close to being orthogonal. The orthogonalization can be done as follows. Let the SVD decomposition of $\U$ be $\U = \U_0\Sigma_0\mathbf{V}_0^\top$. Then perform another SVD decomposition $\Sigma_0\mathbf{V}_0^\top\Sigma\mathbf{V}^\top = \U_1\Sigma_1\mathbf{V}_1^\top$. Finally, the orthogonalized left and right singular vector matrix would be $\U_0\U_1$ and $\mathbf{V}_1$, respectively, and the singular value matrix would be $\Sigma_1$. In other words $\A\approx \left(\U_0\U_1\right)\Sigma_1\left(\mathbf{V}_1^\top\right)$. It's easy to verify that the computational cost is only $\mathcal{O}\left((m+n)k^2\right)$.

\section{Proof of Theorem~2}

First, note that the CABS algorithm uses only the $k$-dimensional embedding instead of the exact embedding as stated in theorem. The consequence is that the resultant error bound will be loosened by the trailing singular values of the input matrix, as follows
\begin{eqnarray*}
\left\|\A - \C\W^\dagger\R^\top\right\|_F &= &\left\|\A_k - \C\W^\dagger\R^\top + \A_{\overline{k}}\right\|_F\\
&\leq& \left\|\A_k - \C\W^\dagger\R^\top \right\|_F + \left\|\A_{\overline{k}}\right\|_F\\
&\leq& \sqrt{T^cT^r}\left(
\sqrt{3k\c({e}^r + e^c)(\T^r+\T^c)} + k\c\sqrt{e^ce^r}\left\|\W^\dagger\right\|_F\right) + \|\A_{\overline{k}}\|_F\\
&=&\mu\sqrt{e^r + e^c} + \nu\sqrt{e^r\cdot e^c} + \|\A_{\overline{k}}\|_F.
\end{eqnarray*}
where
\begin{eqnarray*}
\mu &=& \sqrt{3k\theta T^cT^r(\T^r+\T^c)},\\
\nu &=& k\theta\sqrt{T^cT^r}\left\|\W^\dagger\right\|_F,
\end{eqnarray*}
and  $\|\A_{\overline{k}}\|_F =\sqrt{\sum_{i = k+1}^{\min(m,n)}\sigma_i^2 }$ is a constant which is the $l_2$-norm of the  $\min{(m,n)}-k$ singular values. In case the singular-value spectrum decays rapidly, this constant can be quite small. In other words the error bound is only slightly loosened in case only (approximate) rank-$k$ embeddings (instead of the exact embeddings) are used for the follow-up sampling.

In the following we will use the updated error bound for the pilot and follow-up sketching, as
\begin{eqnarray*}
\Psi_p = \mu_p\sqrt{e^r_p + e^c_p} + \nu_p\sqrt{e^r_p\cdot e^c_p} + \|\A_{\overline{k}}\|,\\
\Psi_f = \mu_p\sqrt{e^r_f + e^c_f} + \nu_f\sqrt{e^r_f\cdot e^c_f} + \|\A_{\overline{k}}\|,
\end{eqnarray*}
and
\begin{eqnarray*}
\mu_p = \sqrt{3k\theta T^c_pT^r_p(T^r_p+T^c_p)}, &\nu_p = k\theta\sqrt{T^c_pT^r_p}\left\|\W_p^\dagger\right\|_F,\\
\mu_f = \sqrt{3k\theta T^c_fT^r_f(T^r_f+T^c_f)}, &\nu_f = k\theta\sqrt{T^c_fT^r_f}\left\|\W_f^\dagger\right\|_F.
\end{eqnarray*}
Here the sub-index $\{p,f\}$ denotes the pilot and the follow-up step, and all parameters are defined in the same way as in Theorem~1. For example, $T^c_{p,f}$ and $T^r_{p,f}$ are the maximum cluster sizes in the column and row embeddings; $e^c_{p,f}$ and $e^r_{p,f}$ are the encoding errors for the column and row embeddings.
The above relation holds because the random sampling in the pilot step can be deemed as equivalently running on the the rank-$k$ embeddings of the input matrix. This instantly gives the following guarantee
\begin{eqnarray*}
\Delta_e^r = e^r_p - e^r_f \geq 0,\\
\Delta_e^c = e^c_p - e^c_f \geq 0.
\end{eqnarray*}
Here  $\Delta_e^r$ and $\Delta_e^c$ are exactly the drop of the encoding errors achieved by the $k$-means sampling algorithm in the follow-up sampling step. Next, we will show that, the drop of the error bounds from the pilot sketching step to the follow-up sketching step in CABS, can be exactly quantified by the drop of the encoding errors $\Delta_e^r$ and $\Delta_e^c$.

We will also use the inequality $g(x) - g(y) \geq (x-y)\cdot g'(x)$ for the function $g(x) = \sqrt{x}$ and any pair of numbers $x\geq y\geq 0$. Namely $\sqrt{x}-\sqrt{y}\geq (x-y)\frac{1}{2\sqrt{x}}$.
We make the realistic assumption that $\T^r_p=\T^r_f = T^r$, and $\T^c_p=\T^c_f = T^c$,  since as we mentioned, the pilot random sampling and the follow-up $k$-means sampling can be deemed as running on the same, rank-$k$ embedding of the input matrix. Namely the maximum cluster sizes in the two rounds of samplings will be the same. So we can safely write $\mu_p = \mu_f = \mu$.  On other hand, we also assume that $\left\|\W_f^\dagger\right\|_F \leq \left\|\W_p^\dagger\right\|_F$. This is because the follow-up sampling using the $k$-means sampling will pick highly non-redundant rows and columns as the representatives, therefore the norm of the resultant intersection matrix $\|\W^\dagger\|_F$ will typically drop. In other words,
\begin{eqnarray*}
\nu_p\sqrt{e^r_p\cdot e^c_p} - \nu_f\sqrt{e^r_f\cdot e^c_f} &=& k\theta\sqrt{e^r_p\cdot e^c_p}\sqrt{T^cT^r}\left\|\W_p^\dagger\right\|_F - k\theta\sqrt{e^r_f\cdot e^c_f}\sqrt{T^cT^r}\left\|\W_f^\dagger\right\|_F\\
 &\geq & k\theta\sqrt{e^r_p\cdot e^c_p}\sqrt{T^cT^r}\left\|\W_p^\dagger\right\|_F - k\theta\sqrt{e^r_f\cdot e^c_f}\sqrt{T^cT^r}\left\|\W_p^\dagger\right\|_F\\
 &=&k\theta\sqrt{T^cT^r}\left\|\W_p^\dagger\right\|_F\left(\sqrt{e^r_p\cdot e^c_p} - \sqrt{e^r_f\cdot e^c_f}\right)\\
 &\geq & 0.
\end{eqnarray*}So we further bound the difference as follows
\begin{eqnarray*}
\Psi_p - \Psi_f &= &\left(\mu_p\sqrt{e^r_p + e^c_p} -\mu_p\sqrt{e^r_f + e^c_f} \right)+ \left(\nu_p\sqrt{e^r_p\cdot e^c_p} - \nu_f\sqrt{e^r_f\cdot e^c_f}\right)\\
&\geq & \mu \left(\sqrt{e^r_p + e^c_p} -\sqrt{e^r_f + e^c_f} \right) + k\theta\sqrt{T^cT^r}\left\|\W_p^\dagger\right\|_F\left(\sqrt{e^r_p\cdot e^c_p} - \sqrt{e^r_f\cdot e^c_f}\right)\\
&  \geq & \mu\frac{1}{2\sqrt{e^r_p + e^c_p}}\left((e^r_p - e^r_f) + (e^c_p - e^c_f)\right) + k\theta\sqrt{T^cT^r}\left\|\W_p^\dagger\right\|_F\frac{1}{\sqrt{e^r_p\cdot e^c_p}}\left(e^r_p\cdot e^c_p - e^r_f\cdot e^c_f\right)\\
& = & \sqrt{3k\theta T^cT^r(T^r+T^c)}\frac{1}{2\sqrt{e^r_p + e^c_p}}(\Delta^r_e + \Delta^c_e) \\ && + k\theta\sqrt{T^cT^r}\left\|\W_p^\dagger\right\|_F\frac{1}{\sqrt{e^r_p\cdot e^c_p}}\left(\Delta^r_e e^c_f + \Delta^c_e e^r_f + \Delta^r_e\Delta^c_e\right).
\end{eqnarray*}
This completes the proof of Theorem~2.

\end{document}